%% file: npexp.tex
\documentclass[11pt,titlepage]{article}
\pdfoutput=1

\usepackage{techreport}
\usepackage{graphicx} % more modern
\usepackage{subfigure}

\usepackage[round,comma,sort&compress]{natbib}

\usepackage{algorithm}
\usepackage{algorithmic}

\usepackage{hyperref}
\usepackage{xcolor}
\definecolor{dark-red}{rgb}{0.4,0.15,0.15}
\definecolor{dark-blue}{rgb}{0.15,0.15,0.4}
\definecolor{medium-blue}{rgb}{0,0,0.5}
\hypersetup{
    colorlinks, linkcolor={red},
    citecolor={dark-blue}, urlcolor={blue}
}

\usepackage{comment} %for comment
\usepackage{color}
\usepackage{enumitem}
\setlist{nolistsep}
\input{preamble} %for math and notations.
\ifpdf
\graphicspath{{./fig/}{./}}
\else
\graphicspath{{./fig_eps/}}
\fi

\usepackage[T1]{fontenc}

\begin{document}

\title{Estimating Densities with Non-Parametric Exponential Families}

%\author{}

\author{
Lin Yuan$^\ast$, Sergey Kirshner$^\dagger$, Robert Givan$^\ast$\\
{\tt \{yuanl,skirshne,givan\}@purdue.edu}\\
\phantom{p}\\
${}^\ast$ School of Electrical and Computer Engineering\\
${}^\dagger$ Department of Statistics\\
Purdue University, West Lafayette, IN 47907, USA
}

\date{June 20, 2012}

\maketitle

%can asymptotically approximate densities outside of the chosen exponential family, or
\begin{abstract}
We propose a novel approach for density estimation with exponential families for the case when the true density may not fall within the chosen family.  Our approach augments the sufficient statistics with
features designed to accumulate probability mass in the neighborhood of the observed points, resulting in a non-parametric model similar to kernel density estimators.  We show that under mild conditions, the resulting model uses only the sufficient statistics if the density is within the chosen exponential family, and asymptotically, it approximates densities outside of the chosen exponential family.  Using the proposed approach, we modify the exponential random graph model, commonly used for modeling small-size graph distributions, to address the well-known issue of model degeneracy.

\end{abstract}

% Outline
\input{intro}
\input{npexpfamily}

\input{nergm}
\input{experiments}

\input{conclusion}

\section*{Acknowledgements}
The authors thank Anthony Quas (via \url{mathoverflow.net}) and Herman Rubin for their help with proofs of the theoretical results.  This research was supported by the NSF Award IIS-0916686.

\newpage
\bibliography{npexp}
\bibliographystyle{plainnat}

\clearpage

%{\huge\bf Supplemental Materials}
\appendix
\input{appendix_formal}

%\newpage
%\input{supp}

\end{document}

%% file: preamble.tex
\usepackage{xspace}
\usepackage{amsmath}
\usepackage{amsthm}
\usepackage{amssymb}
\usepackage{textcase}
\usepackage{extarrows}
\usepackage{graphicx}

\theoremstyle{plain}% default
\newtheorem{theorem}{Theorem}[section]
\newtheorem{lemma}[theorem]{Lemma}

\newtheorem{corollary}[theorem]{Corollary}
\theoremstyle{definition}

\theoremstyle{remark}
\newtheorem{remark}{Remark}

\DeclareMathOperator{\subjectto}{subj\ to}
\DeclareMathOperator{\var}{var}
\DeclareMathOperator*{\argmin}{arg\,min}          %with limits below
\renewcommand{\vec}[1]{\boldsymbol{#1}}

\newcommand{\normof}[1]{\|#1\|}

\newcommand{\bmat}[1]{\begin{bmatrix} #1 \end{bmatrix}}

\newcommand{\ZZ}{\mathbb{Z}}
\newcommand{\RR}{\mathbb{R}}

\providecommand{\vh}{\ensuremath{\vec{h}}}

\providecommand{\vs}{\ensuremath{\vec{s}}}
\providecommand{\vt}{\ensuremath{\vec{t}}}

\providecommand{\vx}{\ensuremath{\vec{x}}}
\providecommand{\vy}{\ensuremath{\vec{y}}}

\providecommand{\vX}{\ensuremath{\vec{X}}}
\providecommand{\vH}{\mathbf{H}}
\newcommand{\bmath}[1]{\ensuremath{\vec{#1}}}
\newcommand{\inner}[2]{\ensuremath{\left<#1,#2\right>}}                   % inner product
\newcommand{\iid}{\ensuremath{\stackrel{i.i.d}{\sim}}}

\newcommand{\sC}{\ensuremath{\mathcal{C}}}
\newcommand{\sB}{\ensuremath{\mathcal{B}}}                                 % the ball
\newcommand{\sF}{\ensuremath{\mathcal{F}}}
\newcommand{\sX}{\ensuremath{\mathcal{X}}}
\newcommand{\sQ}{\ensuremath{\mathcal{Q}}}
\newcommand{\sI}{\ensuremath{\mathcal{I}}}
\newcommand{\sH}{\ensuremath{\mathcal{H}}}
\newcommand{\EF}{\ensuremath{\mathcal{EF}}}
\newcommand{\NEF}{\ensuremath{\mathcal{NEF}}}
\providecommand{\Gn}{\ensuremath{\mathcal{G}_n}}
\providecommand{\Gnk}[1]{\ensuremath{\mathcal{G}_{#1}}}

\providecommand{\vlambda}{\ensuremath{\vec{\lambda}}}        % canonical param

\providecommand{\vtheta}{\ensuremath{\vec{\theta}}}
\providecommand{\ulambda}{\ensuremath{\underline{\vlambda}}}
\providecommand{\uhfunc}{\ensuremath{\underline{\vh}}}
\providecommand{\vmu}{\ensuremath{\vec{\mu}}}                % mean parametrize
\providecommand{\vkappa}{\ensuremath{\vec{\kappa}}}          % mean kernel
\DeclareMathOperator{\rint}{rint}
\DeclareMathOperator{\rbd}{rbd}
\DeclareMathOperator{\conv}{conv}
\DeclareMathOperator{\diff}{d}
\DeclareMathOperator*{\sumsum}{\sum\sum}      % double summation

\newcommand\KDE{{\rm KDE}\xspace}
\newcommand\KDES{{\rm KDEs}\xspace}
\newcommand\ERGM{{\rm ERGM}\xspace}
\newcommand\ERGMS{{\rm ERGMs}\xspace}
\newcommand\NERGM{{\rm MPERGM}\xspace}
\newcommand\NERGMS{{\rm MPERGMs}\xspace}
\newcommand\MAXENT{{\rm MaxEnt}\xspace}
\newcommand\MCMC{{\rm MCMC}\xspace}
\newcommand\NNZ{{\rm NNZ}\xspace}

%% file: intro.tex
\section{Introduction}

The problem of density estimation is ubiquitous in machine learning and statistics.  A typical approach would assume a parametric family for the distribution from which the observed data is drawn and estimate the parameters by fitting them to the data.  Among the parametric families, {\em exponential families} play a prominent role, as maximum likelihood estimation from complete data for exponential families is asymptotically unbiased, consistent, and efficient \citep[][Chapter 5]{Vaart1998}. Finding the maximum likelihood estimator (MLE) reduces to a convex optimization problem that requires knowing only the sufficient statistics from the data \citep{Barndorff-Nielsen1978,Brown1986}.
When the functional form is not readily available, {\em non-parametric approaches} (e.g., kernel density estimation) provide a convenient alternative by allowing the number of parameters to grow with the available data.  However, if the number of components in the data vectors is large relative to the number of the available data points, non-parametric approaches may suffer from the curse of dimensionality \citep{Bellman57} and overfit.  From the standpoint of a bias-variance tradeoff, a parametric approach can be useful even if the number of observations is not large (as statistics of the data can often be estimated from relatively few samples) but could also be hindered by a misspecification bias if the true distribution falls outside the chosen parametric family.
\begin{comment}
I think the author is nitpicking.  I tweaked the text slightly.
\end{comment}
Conversely, non-parametric approaches can approximate any density with enough data, suffering a high variance when the data is limited or the dimensionality is even moderately large.

%LY: maybe we should consider changing the name "mass indicator function" to "weighting function/unnormalized weighting function", which was used in Parzen's paper...

We propose a novel non-parametric density-estimation approach for exponential families that combines some of the strengths of parametric and non-parametric approaches.  Our approach draws inspiration from kernel density estimators (\KDES), which approximate unknown densities by placing probability mass around the observations, and from the exponential families by imposing global constraints in matching the statistics.  Exponential families are derived by maximizing the Shannon's entropy
\begin{comment}(or minimizing the Kullback-Leibler divergence)
\end{comment}
of the estimated distribution subject to the constraint that the expected values of the chosen statistics (features) with respect to the empirical and the estimated distributions must match.  Our proposed exponential family model imposes additional constraints requiring a small constant probability mass around each example point.  We accomplish this by augmenting the set of given statistics (features) with kernel
\begin{comment}
mass indicator
\end{comment}
functions centered around the observations, so that the expected value for each of these functions represents the probability mass concentrated around an example point.
\begin{comment}
but also the masses of the estimated and the empirical distribution in the neighborhood of each observation.
\end{comment}
%% rlg: Linguistically odd that we get to say it is ``non-parametric'' BECAUSE it has so many *parameters*
%% SK: this is not my definition -- in statistics, non-parametric models means the model which increases the number of parameters with available data
The resulting exponential family model is non-parametric, as each data point has a parameter associated with it.  The objective function for the parameter estimation is convex and contains an $\ell_1$-penalty term for each added parameter. These penalties encourage sparsity by potentially making many of the added parameters vanish. We show that if the true distribution is within the exponential family model with the chosen statistics, then as the sample size increases, all parameters associated with the added local features vanish and our approach converges to the true distribution.  If the true distribution is {\em not} from the chosen exponential family, then, our approach provides a close approximation to the unknown density, comparable to \KDES.

Our work is in part motivated by a problem of learning distributions
over graphs from examples of observed networks, typically from a {\em single} network.  Such data
arises in many domains, including social sciences, bioinformatics, and systems
sciences.  Among the approaches to this problem, one of the perhaps most-studied is the exponential random graph model \citep[\ERGM, or in the social network literature, $p\star$, e.g.,][]{HollandLeinhardt1981,FrankStrauss,WassermanPattison}.
\begin{comment}
The number of available examples is typically very small, possibly
only one instance.  Domain scientists are often interested in
generating other network samples with specified properties similar to
that of the observed sample; one of the approaches is to use the
observations to learn a generative model for graphs, and then draw
samples from this model.  Exponential random graph models (\ERGMS, or $p^\star$) \cite{HollandLeinhardt1981,FrankStrauss,WassermanPattison}
is a popular model used in social sciences.
\end{comment}
\begin{comment}
These models are then used to generate similar graphs, for explaining the behavior of the known network or for further processing with domain knowledge.
\end{comment}
\ERGMS use graph statistics as features to define an exponential
family distribution over all possible graphs with a given number of
nodes.  Such models have the desirable property of learning a
distribution that matches the observed graph statistics.
\begin{comment}
Such problem is similar to density estimation problems in natural, whereas the difficulty is there are few samples available (usually only one network at observation) \cite{WassermanPattison}.
\end{comment}
However, \ERGMS often suffer from issues of \emph{degeneracy}
\citep{Handcock2003b,Rinaldo2009,Lunga2011} manifested in placing most of the probability mass on unrealistic
graphs (e.g., an empty or a complete graph), very dissimilar to the observed graph(s).  As an illustration of our approach, we propose a modification to \ERGMS which alleviates the
above issue of degeneracy in moderate-sized graphs.

The main contributions of this paper are a novel framework for non-parametric estimation of densities with exponential family models that is applicable when the number of data points is relatively small, analysis of of its convergence properties, and a modification of \ERGMS that remedies one of the degeneracy issues.  The paper is structured as follows.  In Section \ref{sec:exp}, we briefly describe the exponential family models.  In Section \ref{sec:npexp} we introduce the features we use to constrain the probability mass around the data points and derive a formulation for a new model from first principles.   In Section \ref{sec:properties}, we derive some of the new model's properties, and then discuss the resulting parameter-estimation optimization problem and our approach to solving it in Section \ref{sec:learning}.  In Section \ref{sec:npergm}, we propose a new model for distributions over networks with a moderate number of nodes.  We explore the properties of our estimator for
1-dimensional densities and for modeling network data via an empirical study in Section \ref{sec:experiments}, and finally discuss our findings and outline possible future directions in Section \ref{sec:conclusion}.

%LY: added section links.

\begin{comment}
In the following sections, we start with a general density estimation problem, and describe the principles for non-parametric exponential family as well as generalized \MAXENT. We then proceed with both exact learning algorithms and approximate learning algorithms based on Markov Chain Monte Carlo (\MCMC) sampling. Section \ref{sec:nergm} will revisit the \ERGM model and discuss its relationship with continuous infinite support exponential family models. We then describe \NERGM. In section \ref{sec:experiments}, we will show the density estimation result with both synthetic and small real-world datasets for a proof of concept, and then show the performance of \NERGM fitting both synthetic small network data and large real-world network data.
\end{comment}

%% file: npexpfamily.tex
\section{Exponential Family}\label{sec:exp}
We briefly introduce the exponential family of distributions before describing our contribution, a non-parametric exponential family.

Suppose $\vX$ is a vector of random variables with support $\sX
\subseteq \RR^{m}$.  A distribution for $\vX$ belongs to the exponential
family of distributions with sufficient statistics
$\vt:\sX\to\sH\subseteq\RR^d$, if its probability density has a functional form:\footnote{For notational convenience, we denote $\vX=\vx$ by $\vx$.}
\begin{comment}
\footnote{For simplicity of presentation, we assume the distribution is absolutely continuous and the density exists every on the support.}
\end{comment}
\begin{equation}
\begin{split}
f^E\left(\vx\vert\vlambda\right) &= \frac{1}{Z\left(\vlambda\right)}q\left(\vx\right)\exp\left<\vlambda,\vt\left(\vx\right)\right>\mbox{ where }\\
Z\left(\vlambda\right) &= \int_{\sX}q\left(\vx\right)\exp\inner{\vlambda}{\vt\left(\vx\right)} \diff\vx<\infty
\label{eqn:exp}
\end{split}
\end{equation}
is a partition function, $\vlambda$ is a vector of canonical
parameters, $q:\sX\to\RR$ is a base measure, and $\inner{\cdot}{\cdot}$ denotes
the Euclidean inner product. We further assume that the exponential family is regular (i.e. the canonical parameter space is open).
%We further assume that the exponential family is minimal (i.e., the components of
%$\vt$ are linearly independent) and full (the set of possible $\vlambda$ is of full rank, $d$).
Assuming $q$ is fixed, let $\EF_{\vt}$ denote the set
of all possible distributions of the form \eqref{eqn:exp} with the set
of sufficient statistics $\vt$.
\begin{comment}
Each element of $\vt$ can be indexed by a vector $\vlambda$ so that $\int_{\sX}q\left(\vx\right)\exp\inner{\vlambda}{\vt\left(\vx\right)} \diff\vx<\infty$.
\end{comment}

Given samples $\vx^{1:n}\triangleq\left(\vx^1,\dots,\vx^n\right) \iid f$ where
$f:\sX\to\RR$ is an unknown density with the same support as $q$.  Let
$\hat{f}_n:\sX\to\RR$ be the empirical distribution for $\vx^{1:n}$,
$\hat{f}_n\left(\vx\vert \vx^{1:n}\right)=\frac{1}{n}\sum_{i=1}^n\delta\left(\vx^i\right)$
where $\delta\left(\vx\right)$ is a Dirac delta function.  Exponential
families can be obtained as a solution to the optimization problem of
minimizing the relative entropy subject to matching the moment constraints of the empirical and the estimated distributions:
\begin{align}
f^E_n\left(\vx\right) &= \argmin_{f^E\in\mathcal{F}} KL\left(f^E \parallel q\right)\ \subjectto\label{eqn:maxent}\\
E_{f^E_n\left(\vx\right)}\left[\vt\left(\vx\right)\right] &= E_{\hat{f}_n\left(\vx\vert\vx^{1:n}\right)}\left[\vt\left(\vx\right)\right].\label{eqn:moments}
\end{align}
\begin{comment}
Consider a distribution from the natural exponential family with support $\sX$ with a vector of sufficient statistics $\vt:\sX\to\sH\subseteq\RR^d$, and assume that the components of $\vt$ are linearly independent.  An exponential family density $f:\sX\to\RR$ with sufficient statistics $\vt$ with base measure $q:\sX\to\RR$ is obtained by solving the optimization

For simplicity, we assume that $q\left(\vx\right)\propto 1$ for $\vx\in\sX$ (even though $q$ itself may be an improper density); the results in the paper generalize to non-uniform base measures.
\end{comment}
A distribution
$f^E_n\left(\vx \vert \hat{\vlambda}_n\right)\in\EF_{\vt}$
satisfying \eqref{eqn:moments} can be found by maximizing the
log-likelihood
$l\left(\vlambda\right)=\left<\vlambda,\frac{1}{n}\sum\nolimits_{i=1}^n\vt\left(\vx^i\right)\right>-\log
Z\left(\vlambda\right)$, and provided
%LY: should be g_n instead of \hg_n
$\vt^\star=E_{\hat{f}_n\left(\vx\vert\vx^{1:n}\right)}\left[\vt\left(\vx\right)\right]\in\rint\left(\conv\left(\sH\right)\right)$,
%LY: should be convex hull of H?
a maximum likelihood estimate (MLE) $\hat{\vlambda}_n$ satisfying \eqref{eqn:moments} exists \citep{Wainwright2008}, and will be unique if $\EF_{\vt}$ is minimal \citep{Brown1986}.
If $f\in\EF_{\vt}$, then $\hat{\vlambda}_n\stackrel{p}{\to}\vlambda$ \citep{Vaart1998}.
%If $f\in\EF_{\vt}$, then the MLE $\hat{\vlambda}_n$ for large $n$ is normally distributed around the true values of the parameters $\vlambda$ with variance of the error decreasing linearly with $n$.
%We adopt the following terminology:

% \begin{itemize}
% \item $f$: true density
% \item $\hat{f}_n$: empirical density from $n$ samples
% \item $\tilde{f}_n$: density from the $\mathcal{NEF}_{\vt}$
% \item $f_n$: KDE estimator from $n$ samples
% \item $g_n=E_{f}\left[f_n\left(\cdot\vert\vx^1,\dots,\vx^n\right)\right]$: averaged KDE estimator over all possible sets of kernel centers
% \item $\tilde{g}_n = E_{f}\left[\tilde{f}_n\left(\cdot\vert\vx^1,\dots,\vx^n\right)\right]$: averaged NPEXP estimator over all possible observed data points
% \end{itemize}

However, if the true distribution does not fall within the chosen exponential family, $f\not\in\EF_{\vt}$,
\begin{comment}
or if the number of samples $n$ is small,
\end{comment}
the estimated model may provide a poor approximation to the true density.   As will be illustrated in Section \ref{sec:npergm}, for the case of discrete random vectors $\vX$ from the exponential family with a bounded support $\sH$, finding the MLE under the wrong modeling assumption $f\in\EF_{\vt}$ may assign very little probability mass to the observed samples $\vx^{1:n}$.
\begin{comment}
Figure \ref{fig:degen} illustrates that for a simple discrete natural exponential family using one feature $h\left(x\right)=x$ with support $\sX=\{1,2,3\}$.  Suppose we have a single observation (a common setting for modeling of graphs, see Section \ref{sec:npergm}); if we observe $x=2$, no matter what parameter we choose, $\{x=2\}$ will not be the mode even though it is the {\em only} observed point.
%SK: The example I initially had placed arbitrarily small amount of mass on $x=2$.  Should we resurrect it?
%rlg: I think the example as is will do fine.
\begin{figure}
\centering
%\input{tikz_fig1}
\includegraphics[width=4cm]{tikz_fig1}
\caption{Illustration of degeneracy in on a natural exponential family with a discrete support $\sX={1,2,3}$. \label{fig:degen}}
\end{figure}
\end{comment}

\section{Non-parametric Exponential Family}\label{sec:npexp}
%% rlg: most of these paragraphs (up to 3.1) could be deleted as redundant to the intro, if necessary, or merged into the intro removing a lot.
In this section, we propose a new family of distributions, a modification to the exponential family $\EF_{\vt}$.  Our proposed approach modifies the set of features so that the estimated density (or a probability mass function for discrete vectors) places approximately the same amount of mass around each sample $\vx^i,\ i=1,\dots,n$ as the empirical distribution.  This approach allows using exponential family models to approximate distributions outside of the exponential family (e.g., mixtures, heavy-tailed distributions).  This approach can also be used to avoid degeneracy in cases where the set of features is poorly chosen (e.g., modeling of graphs with \ERGMS).
\begin{comment}
%the resulting solution from the exponential family places sufficient amount of probability mass around each sample $\vx_i,\ i=1,\dots,n$.
We consider the problem of density estimation for natural exponential family models when only few samples are available. Among possible difficulties of estimation in this scenario, we focus on the following two. One, for discrete exponential families with finite support, due to small sample size, relatively little mass is placed on the observed examples. An example of a problem where this issue is particularly important is modeling of social networks with exponential random graph models (\ERGMS), and we examine this model later. Two, if the true model responsible for the data is not from the exponential family, even with larger sample size the estimated density will not place on the examples, for example, in the tails of the distribution for heavy-tailed distributions.
\end{comment}

%% \begin{wrapfigure}[15]{r}{0.8\textwidth}
%% \begin{center}
%% \includegraphics[width=1.\linewidth]{rexp_n1000_e0dot001_k1dot5_a0}
%% \end{center}
%% \caption{Density estimation from samples from a $t$-distribution.
%%   Black {\tt x}'s are samples; the black solid line is the true
%%   density, the blue dashed line is the fitted Gaussian density, and
%%   the red dashed line is the fitted non-parametric Gaussian density with non-parametric exponential family model with Gaussian kernel with width $1.5$.}
%% \label{fig:fitt}
%% \end{wrapfigure}

\begin{figure}
\centering
\includegraphics[width=0.6\linewidth]{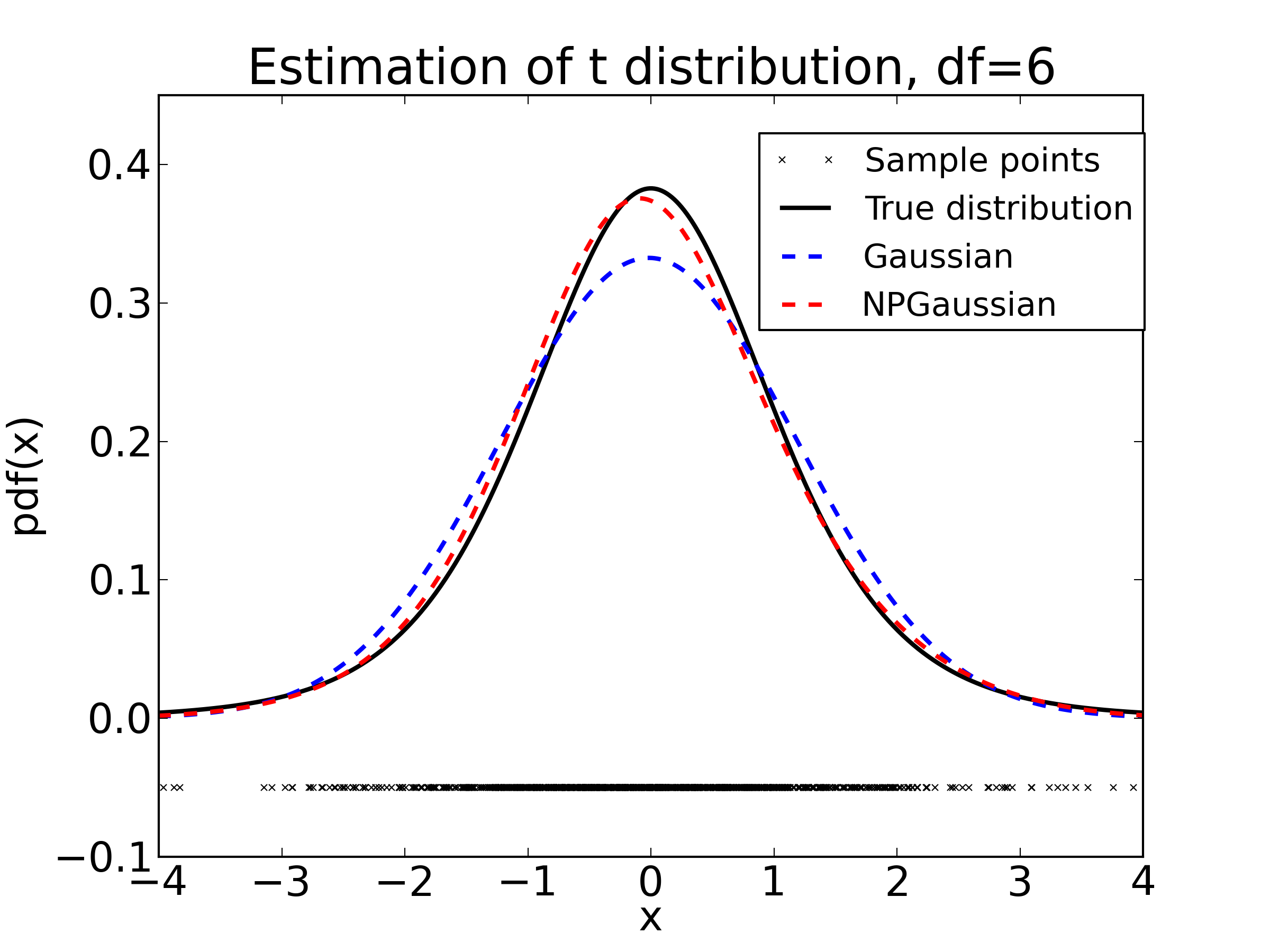}
\caption{Density estimation from samples from a $t$-distribution.
  Black {\tt x}'s are samples; the black solid line is the true density, the
  blue dashed line is the fitted Gaussian density, and the red dashed line is the fitted non-parametric Gaussian density with non-parametric exponential family model with Gaussian kernel with width $1.5$.}
\label{fig:fitt}
\end{figure}

\subsection{Motivation}
Suppose a set of samples from an unknown density ``looks'' Gaussian except
perhaps for a few outliers in the tails (Figure \ref{fig:fitt}).
%{\bf FIGURE FROM $t$-DISTRIBUTION.}
Should we fit a Gaussian?  If not, should we use a non-parametric
approach?  Our approach combines both by using the exponential family
with given features (e.g. $\vt\left(\vx\right)=\left(x,x^2\right)$ in
the case of a univariate Gaussian) as a starting point and then adding
features for each data point. It draws inspiration from \KDES \citep[also known as Parzen windows,][]{Parzen1962},
% LY: using f_p to denote Parzen density may be confusing, changed to f_{\KDE}

\begin{align*}
f_n^{\KDE}\left(\vx\vert\vx^{1:n}\right) &= \frac{1}{n}\sum_{i=1}^nK_{\vH}\left(\vx;\vx^i\right)\mbox{ where}\nonumber\\
K_{\vH}\left(\vx;\vx^i\right) &= \left|\vH\right|^{-\frac{1}{2}}K\left(\vH^{-\frac{1}{2}}\left(\vx-\vx^i\right)\right).
\end{align*}

$K$ is a univariate kernel function, a bounded probability density
function on $\mathbb{R}$. $K_{\vH}$ is a multivariate kernel function
with a symmetric positive definite bandwidth matrix $\vH$; in this
paper, we assume $\vH = h^2\mathbf{I}_d$ (assuming $\vx\in\mathbb{R}^d$).
%SK: removed \sigma from the description of K to make the notation consistent throughout
\begin{comment}
where the multi-dimensional kernel $K\left(\vx\ ;\ \vx^i\right)$ is a density with location (center) parameter equal $\vx^i$ and smoothing (tuning) parameter $\sigma>0$ as its bandwidth (scale parameter).
\end{comment}

The uniform kernel is an indicator function on $\left[-\frac{1}{2},\frac{1}{2}\right]$:
\begin{equation*}
K_U\left(x\right) = 1\mbox{ if }\left|x\right|\leq\frac{1}{2},\mbox{ and }0\mbox{ otherwise,}
%\begin{cases}\frac{1}{\sigma^d}&\left|\frac{x_j-x_j^i}{\sigma}\right|\leq \frac{1}{2},\ \forall j=1\dots d.\\0&\mbox{otherwise}.\end{cases}
\end{equation*}
where a multi-dimensional version is a weighted indicator function for the hypercube %LY: changed square to cube
centered at $\vx^i$ with each side equal to $\frac{h}{2}$.  Most other kernels used with \KDE are smooth approximations of $K_U$, e.g., Gaussian kernel $K_{\mathcal{N}}\left(x\right) = \frac{1}{\sqrt{2\pi}}e^{-x^2/2}$.
\KDE matches the mass around each data point (weighted according to the kernel) to that of the empirical distribution.  Since the empirical distribution approaches the true distribution as $n$ increases, the accuracy of \KDE approximation improves with the increase in the number of data points and the decrease of bandwidth parameter $h$.   The resulting representation however requires keeping all of the observations as parameters and requires exponentially many data points in the dimension $d$ to approximate the underlying density well.

\subsection{Formulation}
Our approach preserves the mass around each data point  by introducing additional moment constraints.  Let $\sB\subseteq \sX$ be a region in the support of $\vX$, and let $\sI_{\sB}\left(\vx\right)=1$ if $\vx\in\sB$, and $\sI_{\sB}\left(\vx\right)=0$ otherwise denote an indicator function for $\sB$.  Given metric space $\left(\sX,\sigma\right)$, let $\sB_i=\left\{\vx\in\sX:\sigma\left(\vx^i,\vx\right)\leq \varepsilon\right\}$ be an $\varepsilon$-neighborhood of $\vx_i$.  Then the probability mass for density $f$ in the $\varepsilon$-neighborhood $\sB_i$ of $\vx_i$ is $P\left(\sB_i\right)=E_{f}\left[\sI_{\sB_i}\right]$.  We propose adding constraints to \eqref{eqn:moments} which would approximately match the probability masses for $\sB_i$ ($i=1,\dots,n$) between the empirical and the estimated distributions ($\hat{f}_n$ and $f_n$ respectively):
\begin{equation}
\left|E_{\hat{f}_n\left(\vx\vert\vx^{1:n}\right)}\left[\sI_{\sB_i}\left(\vx\right)\right]-E_{f_n\left(\vx\vert\vx^{1:n}\right)}\left[\sI_{\sB_i}\left(\vx\right)\right]\right|\leq \beta_i,\label{eqn:box}
\end{equation}
where $\beta_i\geq0$ determine how closely the masses should match.  Similar to \KDES, $\sI_{\sB_i}$ in \eqref{eqn:box} may be replaced with a multidimensional kernel $K_\vH\left(\vx^i;\vx\right)$, which assigns decaying importance of mass away from the center (e.g., a smoothed version of $\sI_{\sB_i}$). We will use $t_a^i \triangleq K_{\vH}\left(\vx^i ; \vx\right)$, and use $\vt_a\left(\vx\right)\triangleq\left[t_a^1\left(\vx\right),\dots,t_a^n\left(\vx\right)\right]$ to augment the statistics $\vt$ in estimating densities.\footnote{We omit $h$ from $t_a^i$ for the simplicity of notation.  It is a tuning parameter that may be set globally for all $i=1,\dots,n$.} In addition to the canonical parameters $\vlambda$ for sufficient statistics, we add {\em augmented} parameters $\vlambda_a$ for the augmented statistics $\vt_a\left(\vx\right)$.

Our proposed density approximation ($f^{NE}_n\left(\vx\right)$) is a solution to
%% rlg: it would be nice to have the equation (5) defining the estimated distribution directly parallel to eqn (2) in form
%%
\begin{equation}\label{eqn:nexp}
\begin{split}
f_n^{NE}\left(\vx\vert\vx^{1:n}\right) &= \argmin_{f^{NE}\in\sF}KL\left(f^{NE}\parallel q\right)\ \subjectto\\
E_{f_n^{NE}\left(\vx\vert\vx^{1:n}\right)}\left[\vt\left(\vx\right)\right] &=
E_{\hat{f}_n\left(\vx\vert\vx^{1:n}\right)}\left[\vt\left(\vx\right)\right],\\
\left|E_{f_n^{NE}}\left[t_a^i\left(\vx\right)\right]\right.&-\left.E_{\hat{f}_n}\left[t_a^i\left(\vx\right)\right]\right|
\leq \beta_i,\ i=1,\dots,n.
%& E_{f}\left[\vt_{\epsilon,i}\left(\vx\right)\right] \geq
%\delta_i,\ i=1,\dots,n.
\end{split}
\end{equation}
$f_n^{NE}$ falls within the generalized \MAXENT framework \citep{Dudik2007}:
\begin{equation}\label{eqn:nexp form}
\begin{split}
f\left(\vx\right) &= \frac{1}{Z\left(\vlambda,\vlambda_a\right)}q\left(\vx\right)\exp\left[\inner{\vlambda}{\vt\left(\vx\right)} + \inner{\vlambda_a}{\vt_a\left(\vx\right)}\right]\\
Z\left(\vlambda,\vlambda_a\right) &= \int_{\sX}q\left(\vx\right)\exp\left[\inner{\vlambda}{\vt\left(\vx\right)} + \inner{\vlambda_a}{\vt_a\left(\vx\right)}\right]\diff \vx.
\end{split}
\end{equation}
Let $\vs\left(\vx\right) \triangleq \left(\vt\left(\vx\right),\vt_a\left(\vx\right)\right)$ and $\vtheta \triangleq \left(\vlambda,\vlambda_a\right)$ be a combined set of statistics and parameters, respectively, for the augmented model.  A specific set of parameter values for the distribution in \eqref{eqn:nexp form} satisfying the constraints in \eqref{eqn:nexp} can be found by maximizing the penalized log-likelihood
\begin{equation}
l\left(\vtheta\right) = \inner{\vtheta}{\frac{1}{n}\sum_{i=1}^n\vs\left(\vx^i\right)} -\ln Z\left(\vtheta\right) - \sum_{i=1}^n\beta_i\left|\lambda_a^i\right|.\label{eqn:genll}
%&\quad -\ln Z\left(\vlambda,\vlambda_{\epsilon}\right) - \sum_{i=1}^n\delta_i\left|\lambda_i\right|.\label{eqn:genll}
%\inner{\bmat{\vlambda\\\vlambda_{\epsilon}}}{\bmat{\frac{1}{n}\sum_{i=1}^n\vt\left(\vx^i\right)\\\vt_{\epsilon}\left(\vx\right)}} - \ln Z\left(\vlambda,\vlambda_{\epsilon}\right) - %\sum_{i=1}^n\frac{1-\delta_i}{2}\left|\lambda_{\epsilon,i}\right|.
\end{equation}
Note that the distribution $f_n^{NE}$ \footnote{$f_n^{NE}\left(\vx\vert\vtheta,\vx^{1:n}\right)$'s functional form depends on both $\vx^{1:n}$ and $\vtheta$, for convenience, we sometimes omit $\vtheta$ and $\vx^{1:n}$ in notation when referring to $f_n^{NE}$.} satisfying the constraint in (\ref{eqn:nexp}) will always exist since $\hat{f}_n$ satisfies all of the above constraints.  %However, without additional assumptions on $\vs$ some of the parameters may be infinite (e.g. when $\vx^{1:n}$ collapse to a single point, or $\vt\left(\vx\right) \in \rbd\left(\sH\right)$).

\begin{comment}
$m_{min}=\delta>0$, and $m_{max}=1$
$P\left(\sB\right)\in\left[m_{min},m_{max}\right]$ can be equivalently written as a box constraint
\begin{equation}
\left|E_{f}\left[\sI_{\sB}\left(\vx\right)\right]-\frac{m_{min}+m_{max}}{2}\right|\leq \frac{m_{max}-m_{min}}{2}.\label{eqn:box}
\end{equation}
For each sample point $\vx_i$
\end{comment}
\begin{comment}
Adding the constraint \eqref{eqn:box} to the set of constraints \eqref{eqn:moments} will result in a density placing a mass between $\left[m_{min},m_{max}\right]$ in $\sB$ (provided a solution to \eqref{eqn:maxent} with the added constraint exists).  By setting $\sB_i=\left\{\vx\in\sX:d\left(\vx,\vx^i\right)\leq \epsilon\right\}$, $m_{min}=\delta>0$, and $m_{max}=1$, and the set of additional features $\vt_{\epsilon,i}\left(\vx\right)=\sI_{\sB_i}\left(\vx\right)$, the resulting density $f$ would be guaranteed to place at least $\delta_>0$ mass in $\epsilon$-neighborhood of each example $\vx^i$ (if such solution exists).
\end{comment}

We refer to the above class of models as {\em non-parametric exponential family models}, since the number of non-zero parameters $\vtheta=\left(\vlambda,\vlambda_a\right)$ may increase with the number of available data points.  We will denote this family by $\NEF_{\vs}$.  Clearly $\EF_{\vt}\subseteq\NEF_{\vs}$ as all the augmented parameters can be set to $0$.
Let $\hat{\vtheta}_n=\left(\tilde{\vlambda}_n,\tilde{\vlambda}_{a,n}\right)$ be the MLE of $\left(\vlambda,\vlambda_a\right)$ for the case of $n$ samples.
\begin{comment}
to differentiate them from the MLE $\hat{\lambda}_n$ for $\EF_{\vt}$.
\end{comment}
The $\ell_1$-penalty in \eqref{eqn:genll} is known to be sparsity-inducing \citep[e.g.,][]{Bach2011}, so in practice, many of $\tilde{\lambda}_{a,n}^i=0$.

Note that the framework in \eqref{eqn:nexp} allows matching the moments of the estimated distribution to some predetermined vector $\vt^\star\in\rint\left(\conv\left(\sH\right)\right)$ instead of the empirical moments $E_{\hat{f}_n\left(\vx\vert\vx^{1:n}\right)}\left[\vt\left(\vx\right)\right]$, leading to the same functional form for the density \eqref{eqn:nexp form}, but with an additional linear term $\inner{\vlambda}{\vt^\star - \frac{1}{n}\sum_{i=1}^n\vt\left(\vx^i\right)}$ in \eqref{eqn:genll} \citep{Dudik2007}.  Thus, our non-parametric density estimator is capable of satisfying global constraints (matching a set of provided moments, e.g., learning a distribution with a given covariance).  We are not aware of other non-parametric methods with this capability.

\input{theorems}
\input{learning}

%% file: theorems.tex
\subsection{Theoretical Properties}\label{sec:properties}
% rlg: footnote about proofs in appendix inlined as it saves space in paper
\begin{comment}
The non-parametric exponential family has a number of useful properties.% that we prove Appendix \ref{sec:proofs}.
\end{comment}
%
% rlg: why ``a solution to eqn:nexp form''  instead of ``a solution to eqn: nexp'' ??

The proofs appear in the Appendix \ref{sec:proofs}.

\begin{theorem}\label{thm:true}
Suppose a vector of random variables $\vX$ with support on $\sX$ has a density $f\in\EF_{\vt}$ with features $\vt:\sX \to \sH\subseteq\mathbb{R}^d$ and a vector of canonical parameters $\vlambda\in\sC\in\mathbb{R}^d$.
%$ density $f\left(\vx\right)\in\EF_{\vt}$, with domain space $\sX$, feature space $\sH$, and canonical parameters $\vlambda \in \sC$.
Suppose $\vx^1,\dots,\vx^n \triangleq \vx^{1:n}$ is a sequence
of i.i.d. random vectors drawn from $f$.  Let
$f_n^{NE}\left(\vx\vert \hat{\vtheta}_n,\vx^{1:n}\right)\in\mathcal{NEF}_{\vs}$
be the MLE solution of \eqref{eqn:nexp}, $\hat{\vtheta}_n =
\left(\tilde{\vlambda}_n, \tilde{\vlambda}_{a,n}\right)$, with all $\beta_i=\beta>0$, $i=1,\dots,n$.  Assuming
\begin{enumerate}
\setlength{\itemsep}{0em}
\setlength{\parskip}{0pt}
\setlength{\parsep}{0pt}
\item $\sX$ is compact,
\item %$\vt\left(\vx\right)$
$\vt$ is continuous,
\item $\EF_{\vt}$ is a family of uniformly equicontinuous functions w.r.t $\vx$,
%\item In \eqref{eqn:nexp}, we choose $\beta_i = \beta > 0$, $\forall i=1,\dots,n$.
\item Kernel $K$ has bounded variation and has a
  bandwidth parameter $\vH$ such that the series $\sum_{n=1}^{\infty}e^{-\gamma n \left|\vH\right|}$
converges for every positive value of $\gamma$,
\end{enumerate}
then as $n\to\infty$, $\tilde{\lambda}_{a,n}^i\stackrel{p}{\to}0,
\forall i=1,\dots,n$ and $\tilde{\vlambda}_n\stackrel{p}{\to}\vlambda$.

\end{theorem}
%
% rlg: I rephrased the following because it seemed to say the
% augmented parameters would actually eventually BE zero for large
% enough data
%
Intuitively, uniform convergence is required because we need the $n$ additional constraints be satisfied with fixed threshold $\beta$ as long as $n>N$. We use Assumption 2 to relate the pointwise convergence of the MLE $\vlambda$ in regular exponential families to pointwise convergence of $f_n^E\left(\vx \vert \hat{\vlambda}_n\right)$. Assumptions 1 and 3 are further employed to convert the pointwise convergence to uniform convergence of $f_n^E$ by considering $\sX,\sC$ to be subsets of the original regular exponential family.\footnote{For example, Theorem \ref{thm:true} applies to the exponential family $\mathcal{N}(\mu,\sigma^2)$ with $\sigma \in [a,b], \forall b>a>0$, but not if $\sigma^2 \in (0,\infty)$.} Assumption 4 is the requirement used by \cite{Nadaraya1965} for the uniform convergence of \KDE satisfied by common kernels. Upon these uniform convergence results, the augmented constraints will be satisfied by the original MLE estimate $\left[\hat{\vlambda}_n,\vec{0}\right]$. Further, the nature of $\hat{\vlambda}_n$ being a maximum entropy solution guarantees that $\left[\hat{\vlambda}_n,\vec{0}\right]$ is a maximum entropy solution under the additional constraints.

Theorem \ref{thm:true} shows that if the true distribution falls
within the exponential family, then as sample size increases,
the estimated density from the non-parametric exponential family will
have vanishing reliance on the augmented parameters.

\begin{theorem}\label{thm:nef}
Given a probability density function
$f\left(\vx\right):\sX\to\RR$, let $f_n^{NE}\left(\vx\vert
\hat{\theta}_n,\vx^{1:n}\right)\in\mathcal{NEF}_{\vs}$ be a solution satisfying
\eqref{eqn:nexp}. If
\begin{enumerate}
\setlength{\itemsep}{0em}
\setlength{\parskip}{0pt}
\setlength{\parsep}{0pt}
\item $f$ is uniformly continuous on $\sX$,
\item $K_{\vH}\left(\vx\right)$ is uniformly continuous on $\sX$,
\item $\sup_{\vx \in \sX} K_{\vH}\left(\vx\right) < \infty$,
\item $\lim\limits_{\normof{\vx}\to\infty} K_{\vH}\left(\vx\right)\prod\limits_{i=1}^{m}\vx_i = 0$,
\item $\lim\limits_{n\to\infty}\left|\vH\right|^{\frac{1}{2}} = 0$,
\item $\lim\limits_{n\to\infty}n\left|\vH\right|^{\frac{1}{2}} = \infty$,
\end{enumerate}
then $f_n^{NE}\left(\vx \vert \hat{\vtheta}_n,\vx^{1:n}\right) \stackrel{p}{\to} f\left(\vx\right)$ pointwise on $\sX$.
\end{theorem}

Assumptions 3-6 are required for pointwise convergence of \KDE \citep{Parzen1962} at specific points $\vx^{1:n}$. We then extend this pointwise convergence from $\vx^{1:n}$ to $\sX$, considering the probability of sampling a new $\vx\in\sX$ far away from existing $\vx^{1:n}$ under the true density $f$.  The monotone convergence theorem and the uniform continuity assumptions (1 and 2) lead to the pointwise convergence $f_n^{NE}\stackrel{p}{\to}f$ on $\sX$.

Theorem \ref{thm:nef} indicates the weak consistency of the non-parametric exponential family density estimator.
% to the true distribution $f$ which may be {\em outside} of the exponential family $\EF_{\vt}$.
Thus our proposed non-parametric approach can be used to approximate densities which are not from exponential families.

\begin{comment}(i.e. at continuity point on
$\sX$ in probability).
\end{comment}
%\paragraph{What happens when the features do not define a family containing the true distribution.}

%% file: learning.tex
\subsection{Estimating Parameters for Non-Parametric Exponential Families}\label{sec:learning}
%1.Motivation of MaxEnt extra constraint
%2.MaxEnt formulation
%3.Derivation of nergm
%4.Learning
%5.Sampling

\begin{comment}
For the generalized non-parametric exponential family density estimation problem \eqref{eqn:npden}, there is a choice for regularization type for the natural parameters as well as artificial parameters. Generally speaking, there is always a solution given: 1) we don't use any kind of regularization and 2) pick the $\vt_{\epsilon}$ to be kernel functions and 3) only have first-order moment for $\vmu$ and 4) $\vmu\in\rint\left(\sX\right)$. Since the density estimation provided by \KDE is such a solution. However, if $\sX$ has an infinite support, then we may not guarantee the partition function $Z\left(\vlambda,\vlambda_{\epsilon}\right)$ exists or to guarantee the finite variance of the learned distribution. This problem will be discussed in section \ref{sec:nergm}. For simplicity, within this section, we assume $Z\left(\vlambda,\vlambda_{\epsilon}\right)$ always exists.
\end{comment}
%As was previously pointed out, a solution to \eqref{eqn:nexp} always exists although it may require some of the parameters to become infinite.  The log-likelihood function \eqref{eqn:genll} is concave, but due to the presence of absolute values, it is not smooth.  Still,
Recently there have been a number of methods developed for optimization of convex non-smooth functions, some of them specifically aimed at log-linear problems such as \eqref{eqn:genll} \citep[e.g.,][]{Bach2011,WuLange2008,Shalev-ShwartzTewari2011}.  We employed a coordinate descent algorithm similar to the SUMMET algorithm of \cite{Dudik2007} (see Algorithm \ref{alg:coord}), primarily, due to its simplicity.  Other possible approaches can be employed as well and may end up more efficient for this formulation.

\begin{comment}
We initially assume that the gradient for the smooth part of $l\left(\vtheta\right)$
\begin{equation*}
l_s\left(\vtheta\right) = \inner{\vtheta}{\frac{1}{n}\sum_{i=1}^n\vs\left(\vx^i\right)} - \ln Z\left(\vtheta\right)
\end{equation*}
can be computed in closed form:
\begin{equation*}
\nabla_{\vtheta}l_s\left(\vtheta\right) = E_{q_n}\left[\vs\left(\vx\right)\right] - E_{f\left(\cdot\vert \vtheta\right)}\left[\vs\left(\vx\right)\right].
\end{equation*}
\end{comment}
\begin{comment}
To optimize \eqref{eqn:genll} with respect to $\vtheta$, we employ a coordinate descent algorithm similar to the SUMMET algorithm \cite{Dudik2007} (see Algorithm \ref{alg:coord}).
\end{comment}
The proposed algorithm iterates between optimizing canonical parameters $\vlambda$ (by setting $E_{\hat{f}_n}\left[\vt\left(\vx\right)\right] = E_{f^{NE}_n\left(\vx\vert\vtheta^{(k)}\right)}\left[\vt\left(\vx\right)\right]$) and sequentially optimizing the augmented parameters $\vlambda_a$ so that the Karush-Kuhn-Tucker conditions \citep[e.g.,][]{NocedalWright06} are satisfied:
\begin{equation*}
E_{\hat{f}_n}\left[t_a^i\left(\vx\right)\right] - E_{f^{NE}_n\left(\vx\vert\vtheta\right)}\left[t_a^i\left(\vx\right)\right] \in \begin{cases}\left\{\beta_i\right\}&\lambda_a^i>0,\\\left\{-\beta_i\right\}&\lambda_a^i<0,\\
\left(-\beta_i,\beta_i\right)&\beta_a^i=0.\end{cases}
\end{equation*}

\begin{comment}
We will first introduce a coordinate descent algorithm for learning both natural and artificial parameters, given $E_f\left(\left[\vt\left(\vx\right),\vt_{\epsilon}\left(\vx\right)\right]\right)$ can be calculated in closed form. We will use no regularization on $\vlambda$ and $\ell_1$ regularization on $\vlambda_{\epsilon}$, because this choice will produce a sparse artificial parameter set, and keep the shape of natural exponential family distribution as much as possible. For other choices of regularization type, please see \cite{Dudik2007}.

If we use a box constraint with width $\beta$ for the artificial parameters: $\forall i=1\dots n, \left\vert E_{f}\left[\vt_{\epsilon,i}\left(\vx\right)\right]-\vkappa^i\right\vert<\beta$. We can rewrite the dual form of \eqref{eqn:npden} as
\begin{equation}
\begin{array}{ll}
\max_{\left[\vlambda,\vlambda_{\epsilon}\right]} l\left(\vlambda,\vlambda_{\epsilon}\right) & =\inner{\bmat{\vlambda\\\vlambda_{\epsilon}}}{\bmat{\vmu\\\vkappa}} - \ln Z\left(\vlambda,\vlambda_{\epsilon}\right) - \sum_{i=1}^n\beta\left|\lambda_{\epsilon,i}\right|.
\end{array}
\label{eqn:npdendual}
\end{equation}

To solve \eqref{eqn:npdendual}, a coordinate descent algorithm similar to the SUMMET algorithm \cite{Dudik2007} is employed (see Algorithm \ref{alg:coord}). We use steepest descent for the natural parameter set and use implicit/explicit update for the artificial parameter set.
\end{comment}

% rlg: minor suggestion...make end-of-line punctuation scheme systematic (perhaps no punctuation at end of line)
\begin{algorithm}[t]
\caption{Non-Parametric Exponential Family Coordinate Descent}\label{alg:coord}
\begin{algorithmic}
%\STATE \textbf{INPUT:} Samples $\vx^1,\dots,\vx^n \iid g$ where $g:\sX\to\RR$ is an unknown distribution. $\vt:\sX\to\sH$ as the sufficient statistics. Weighting function $h_{a}\left(\vx,\vy\right)=K(\vx ; \vy)$. $\ell_1$ regularization parameter $\vec{\delta}$
\STATE \textbf{INPUT:} Samples $\vx^1,\dots,\vx^n\in\mathbb{R}^d$, sufficient statistics $\vt:\sX\to\sH$, augmented features $t_{a}^i:\sH\to\mathbb{R}$, $i=1,\dots,n$,  $\ell_1$ regularization parameters $\vec{\beta}
$\STATE \textbf{OUTPUT:} MLE $\bmath{\theta}=\left(\lambda^1,\dots,\lambda^d,\lambda_a^1,\dots,\lambda_a^n\right)$
\STATE Initialize $\bmath{\theta}^{\left(0\right)}$
\STATE Compute the sufficient statistics $E_{\hat{f}_n\left(\vx\vert\vx^{1:n}\right)}\left[\vt\left(\bmath{x}\right)\right]$
%\left(Calculate the mean parametrization $\left[\vlambda,\vlambda_a\right]\in\RR^{d+n}$
\REPEAT%{$t=0,2 \dots $}
\STATE iteration $k=k+1$, $\bmath{\theta}^{(k)}=\bmath{\theta}^{(k-1)}$
\FOR{$i=1,\dots,d$}
\STATE $g_{i}^{(k)} = E_{\hat{f}_n}\left[t^i\left(\vx\right)\right]-E_{f_n^{NE}\left(\vx\vert\theta^{\left(k\right)}\right)}\left[t^i\left(\vx\right)\right]$
\STATE Perform line search along $g_{i}^{(k)}$ to update $\lambda^{i,(k)}$%\lambda^{i,(t+1)} = \lambda^{i,(t)} + \alpha g_{i}^{(t)}/\left|g_{i}^{(t)}\right|$, where $\alpha$ is a line-search step size.
\ENDFOR
\FOR{$j=1 \dots n$}
\STATE Solve two equations for $\lambda_a^j$ ($\lambda_a^{j,-}$ and $\lambda_a^{j,+}$, respectively):
\STATE $E_{f^{NE}_n\left(\vx\vert\vtheta^{(k)}\right)}\left[t_{a}^{j}\left(\vx\right)\right] = E_{\hat{f}_n}\left[t_a^j\left(\vx\right)\right]-\beta_j$
\STATE $E_{f^{NE}_n\left(\vx\vert\vtheta^{(k)}\right)}\left[t_{a}^{j}\left(\vx\right)\right] = E_{\hat{f}_n}\left[t_a^j\left(\vx\right)\right]+\beta_j$
%$E_f\left[\vt_{a}^{j}\left(\vx\right) \vert \lambda_{a}^{j,+}\right]=E_{\hg_n}\left[\vt_a^j\left(\vx\right)\right]+\delta_j$
%either with explicit or implicit update. Let the solution be $\vlambda_{\epsilon,j}^-$.
%, for $\lambda_{\epsilon,j}$, either with explicit or implicit update. Let the solution be $\vlambda_{\epsilon,j}^+$.
\STATE
\STATE choose $\lambda_{a}^{j,(k)}=\lambda_{a}^{j,-}$ if $\lambda_{a}^{j,-} > 0$
\STATE choose $\lambda_{a}^{j,(k)}=\lambda_{a}^{j,+}$ if $\lambda_{a}^{j,+} < 0$
\STATE choose $\lambda_{a}^{j,(k)}=0$ otherwise
\ENDFOR
%\STATE Run until convergence
%stop the loop if $E_{f}\left[\left[\vt\left(\vx\right),\vt_{\epsilon}\left(\vx\right)\right]-\left[\vmu,\vkappa\right]\right]\approx \vec{0}$ where the $\approx$ means within a small tolerance of satisfying the $\ell_1$ and mean natural parametrization constraint.
\UNTIL {convergence}
\STATE return $\bmath{\theta}^{(k)}$
\end{algorithmic}
\end{algorithm}

Algorithm \ref{alg:coord} belies the inherent difficulty of: (1) calculation of the partial derivative $g_{i}^{(k)}$, and (2) an implicit search procedure to update $\vlambda_a^{j,(k)}$, both involve calculating intractable integrals. If the support is low-dimensional and the mass is contained in a small volume, then the partition function (and thus the gradient) can be computed by numerical integration (quadrature).  Alternatively, a common approach to MLE with an intractable partition function $Z\left(\vtheta\right)$ is Markov Chain Monte Carlo MLE \citep[MCMC-MLE,][]{Geyer1992}. For example, the time complexity at each iteration $k$ is $O(Sn^2)$, where $S$ is the number of Monte-Carlo samples we choose to use. However, we believe developments in optimization \citep[e.g.]{Bach2011,Shalev-ShwartzTewari2011} will help us find an efficient solution.

%% file: nergm.tex
\section{Application to Modeling of Graphs}\label{sec:npergm}
In this section, we turn our attention to a problem of learning a distribution over $\mathcal{X}=\mathcal{G}_n$, a set of undirected graphs with $n$ vertices and no self-loops, from a single observed instance $G^\star\in\mathcal{G}_n$, an important branch in the analysis of social networks because of complicated relational structure \citep[e.g.][]{Goodreau2007}.
\begin{comment}
The problem of interest is to estimate a distribution over networks from a few observed instances.
\end{comment}
% rlg: might mention the obvious need for strong bias, in this case towards similarity on the provided features?
%
% rlg: the previous line appears to demand at least one citation
% rlg: funny, we now have g_n and \mathcal{G}_n, both distributions
A commonly used approach to this problem which arises in the analysis of social networks is to estimate a distribution using exponential random graph models \citep[\ERGMS, e.g.,][]{WassermanPattison,Handcock2003b,Robins2007a,Robins2007b,WassermanRobins2004}.  This approach however suffers from the model degeneracy, with estimated models placing probability mass on unrealistic graphs (e.g., complete or empty) and away from the observed instance.  We propose a modification to \ERGMS utilizing the non-parametric exponential family approach from Section \ref{sec:npexp} which alleviates the above issue of degeneracy.
\begin{comment}
Here we adapt our non-parametric exponential family from Section \ref{sec:ourapproach} to address issues of degeneracy that arise in using \ERGMS for this purpose.
among other issues \citep[e.g.,][]{ShaliziRinaldo2011}, suffers from
\end{comment}
\subsection{Exponential Random Graph Models}
An \ERGM (or $p^\star$ model) is an exponential family model over $\mathcal{G}_n$ which uses graph statistics as its features\footnote{sufficient statistics for the exponential family}.  These features are typically motivated by the properties of the networks that are of interest to domain scientists (e.g., sociologists), and may include (among other local and global features) the number of edges ($t_e\left(G\right)= \sumsum_{1\leq i<j\leq n}e_{ij}$) and triangles ($t_\triangle\left(G\right)=\operatornamewithlimits{\sum\sum\sum}_{1\leq i<j<k\leq n}e_{ij}e_{ik}e_{jk}$),
%and $k$-stars ($t_{k\star}\left(G\right)$)
%\begin{align}
%t_{e}\left(G\right) &= \sumsum_{1\leq i<j\leq n}e_{ij},\nonumber\\
%t_{\triangle}\left(G\right) &= \operatornamewithlimits{\sum\sum\sum}_{1\leq i<j<k\leq n}e_{ij}e_{ik}e_{jk}\mbox{and}\nonumber\\
%t_{k\star}\left(G\right) &= \sum_{i=1}^n\operatornamewithlimits{\sum\dots\sum}_{\tiny \begin{array}{c}1\leq i_1<\dots\leq i_k\\i_1\neq i,\dots,i_k\neq i\end{array}}\prod_{j=1}^ke_{ii_j},\nonumber
%\end{align}
%\[
%\begin{aligned}
%t_{e}\left(G\right) &= \sumsum_{1\leq i<j\leq n}e_{ij}\\
%t_{\triangle}\left(G\right) &= \operatornamewithlimits{\sum\sum\sum}_{1\leq i<j<k\leq n}e_{ij}e_{ik}e_{jk}
%\end{aligned}
%\]
where $e_{ij}=1$ if there is edge between nodes $i$ and $j$, and $0$ otherwise.  The probability mass for a graph $G\in\mathcal{G}_n$ is defined as
\begin{equation}
\label{eqn:ergm}
\begin{split}
P\left(G\vert\hat{\vlambda}\right) &= \frac{1}{Z\left(\vlambda\right)}\exp\inner{\vlambda}{\vt\left(G\right)},\\
Z\left(\vlambda\right) &= \sum_{G\in\mathcal{G}_n}\exp\inner{\vlambda}{\vt\left(G\right)}.
\end{split}
\end{equation}
The MLE $\hat{\vlambda}$ makes mean statistics of the distribution match that of the observed graph: $E_{P\left(G\vert\vlambda\right)}\left[\vt\left(G\right)\right] =  \vt\left(G^\star\right)$.

\begin{comment}
%SK: move to the experimental section
We consider $\vt\left(G\right)=\left(h_{e}\left(G\right),h_{\triangle}\left(G\right)\right)^T$.
\end{comment}

\subsection{Degeneracy}
\begin{comment}
\ERGM is a natural exponential family with discrete finite support $\sH$ and $\sX=\Gn$ (i.e. all $n$-node graphs). $\forall G \in \Gn$, $\vt(G)$ is usually a set of features one can obtain from a graph $G$. These features may include number of edges ($h_e(G)$), number of triangles ($h_\triangle(G)$), number of $k$-stars ($h_{k\star}(G)$), etc. Please see \cite{Dalton2010} for more details. A distribution over $\Gn$ is described as:

It can be seen as a base measure on $\sH$. For example, the support $\sH$ for $\Gnk{8}$ is illustrated in Figure \ref{fig:g8sh}.
\end{comment}
Let $\mathcal{H}=\left\{\vt\left(G\right): G\in\mathcal{G}_n\right\}$
be the set of all possible values for features.  Even though in theory
if the feature vector for the observed graph is in the relative
interior of the convex hull
$\vt\left(G^\star\right)\in\mbox{rint}\left(\mbox{conv}\left(\mathcal{H}\right)\right)$,
MLE $\hat{\vlambda}$ exists (and is unique if the set of features is
linearly independent or minimal), in practice \ERGMS often suffer
from {\em degeneracy} \citep{Handcock2003b} manifested in one of the
following ways: (1) MLE procedure does not converge due to numerical
instabilities, and (2) MLE is found, but the resulting probability
mass is placed mostly on unrealistic graphs (i.e., empty or complete
graphs) and little mass is placed in the vicinity of the observed
graph (around $\vt\left(G^\star\right)$ in $\mathcal{H}$, c.f. Figure \ref{fig:g8degen}).

\begin{comment}
Given an example graph / several example graphs, degeneracy refer to: 1)either the MLE $\hat{\vlambda}$ doesn't exist ($\vmu \in \rbd(\sH)$); 2)if $\hat{\vlambda}$ exists, the learned model have its mode far from its mean, and the estimated density for the training samples is too low to make re-generating similar samples even possible\cite{Handcock2003b,Rinaldo2009}.
\end{comment}

\begin{figure*}[!tb]
\centering
\begin{minipage}[c]{.12\linewidth}
\centering
\subfigure[The graph with 22 edges and 29 triangles]{\includegraphics[width=1\linewidth]{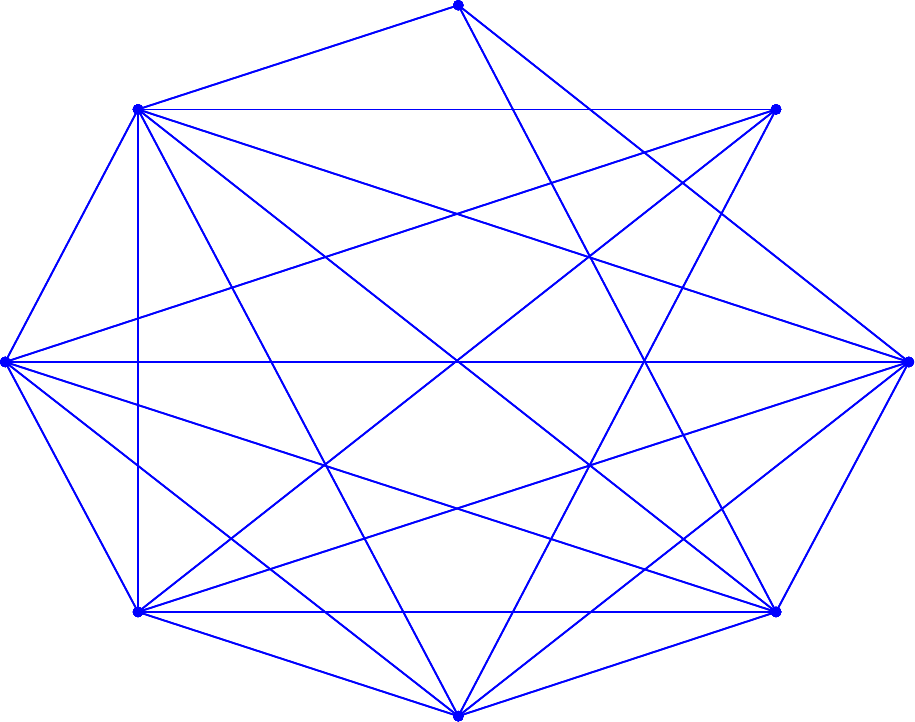}}
\end{minipage}
%\hfill
\begin{minipage}{.36\linewidth}
\subfigure[Probability mass for \ERGM]{\includegraphics[width=1\linewidth]{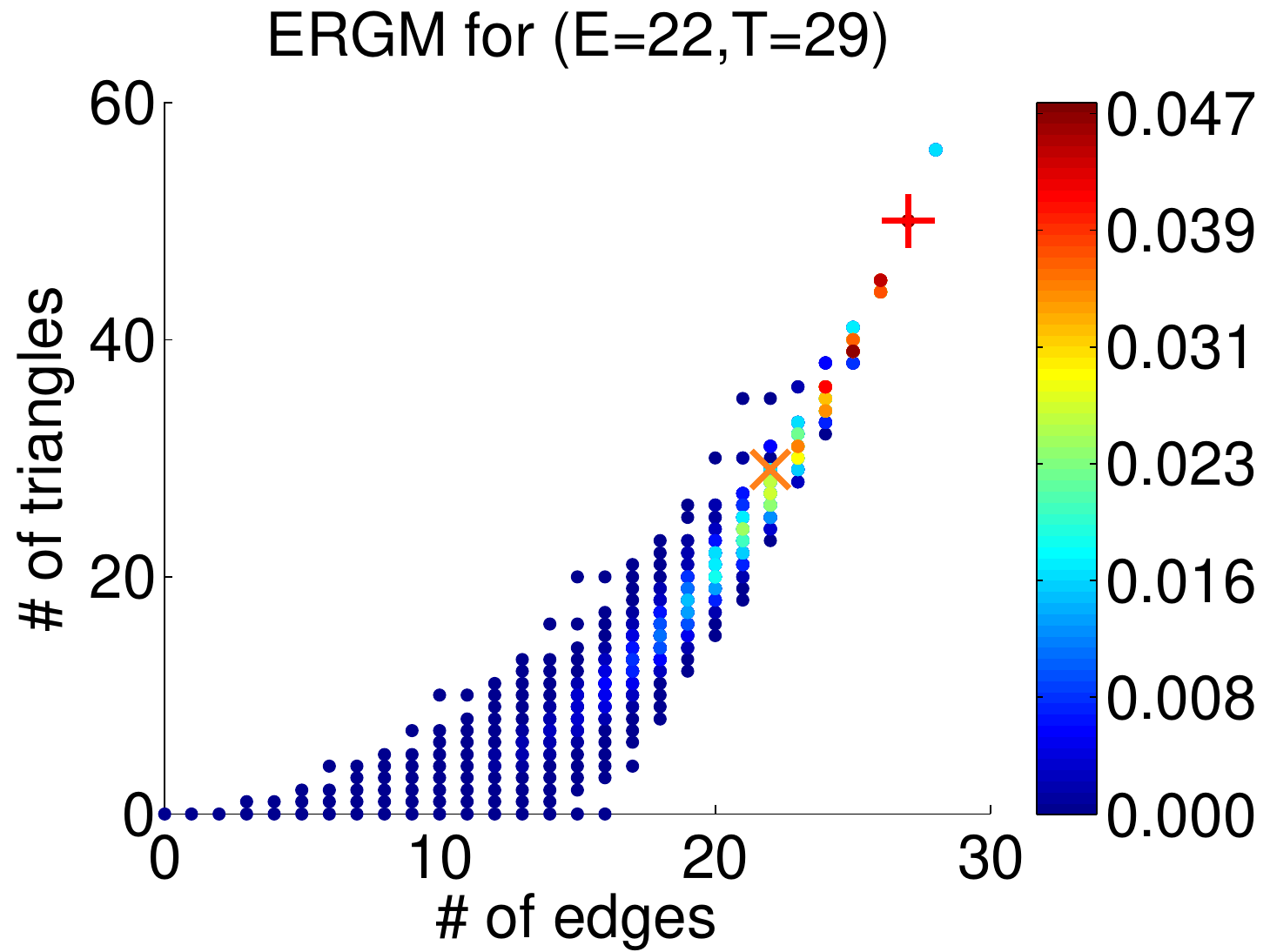}}
\end{minipage}
%\hfill
\begin{minipage}{.36\linewidth}
\subfigure[Probability mass for \NERGM]{\includegraphics[width=1\linewidth]{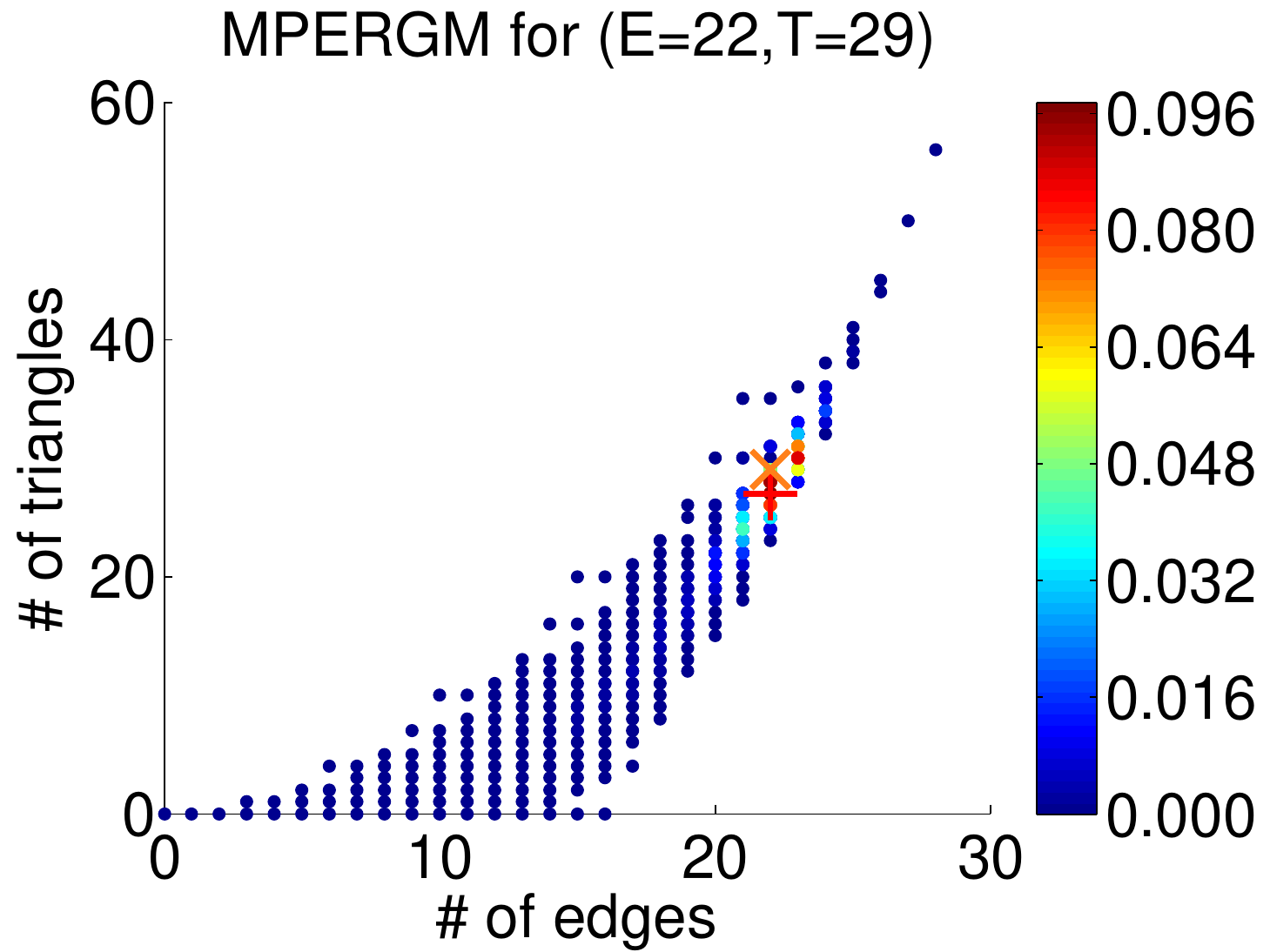}}
\end{minipage}
\caption{Degenerate \ERGM and Non-degenerate \NERGM. The models are
  trained based on the observation $t_{e}(G^\star)=22,t_{\triangle}(G^\star)=29$.
  The orange $\times$ is the observed statistics, and the red $+$ is the mode of the learned model. The color bar on the right from red to blue represents the probability mass changing from high to low.}\label{fig:g8degen}
\end{figure*}

\begin{comment}
\begin{figure}[!h]
\begin{center}
\includegraphics[scale=0.46]{g8ConvexHull3D-Modes-Facets-new}
\caption{View of mode placement on facets of the extended 3D convex hull. Each mode forms a vertex of a given triangle (denoted by black lines) for a facet of the convex hull.}
\label{fig:g8sh}
\end{center}
\end{figure}
\end{comment}

\begin{comment}
\begin{figure}[!ht]
\begin{center}
\subfigure[]{\includegraphics[scale=0.3]{g8Degeneracy}}
\subfigure[]{\includegraphics[scale=0.2]{g8Feature-Pair-Edge-Tri-Modes-Matched-4}}\\
\caption{An example of degenerate \ERGM for $(h_e=16,h_\triangle)=12$ (left) and the total mode-mean mismatch for $\Gnk{8}$ (right).}
\label{fig:totmismatch}
\end{center}
\end{figure}
\end{comment}

We focus on addressing the second type of degeneracy; for more information of the reasons of the first type of degeneracy see \cite{Handcock2003b,Rinaldo2009}.  Several attempts have been made to address the second type of degeneracy issue: \citet{Handcock2008} proposed to use domain knowledge specific feature sets in addition to edge and triangle features; \citet{Hunter2008a,Hunter2006} used curved exponential families for \ERGMS; \citet{Caimo2010} suggested Bayesian \ERGMS, and \citet{Jin2011} devised an estimation procedure on stochastic approximation with varying truncation in parameter space. \citet{Lunga2011} suggested the degeneracy issue for interior points may be due to the bounded support $\sH$, and proposed spherical features for modifying the geometry of $\sH$. In summary, there are two main approaches towards fixing degeneracy: 1) modifying the geometry \citep{Handcock2008,Hunter2008a,Lunga2011}, and 2) limiting exploration in the canonical parameter space \citep{Caimo2010,Jin2011}. Our approach belongs in the first category.

\begin{comment}
but is fundamentally different because we borrow intuition from the theory of dispersion models. \cite{Jorgensen1997}
\end{comment}

\begin{comment}
\subsection{Relation to dispersion models}
According to \cite{Jorgensen1997} page 51, for natural exponential families, the variance function $\var(\vt(\vx))$ uniquely describes the distribution. For example, if $\var(\vt)$ is constant, then we have normal distribution, if $\var(\vt)=\vmu$, we have an exponential distribution in the continuous case and Poisson distribution in the discrete case. For \ERGMS, there simply doesn't exist any constraint on $\var(\vt)$ ($\lambda$ for $\vt^2$ term is $0$). It is similar to have a normal distribution with $\mathcal{N}\left(\mu,\sigma^2\right)$, $\frac{-1}{\sigma^2}=0$. Given $\sX$ is a discrete finite support, $\var(\vt)<+\infty$, however, intuitively, we can see that as the number of nodes grow larger, the number of non-isomorphic graphs in $\Gn$ is growing super-exponentially :$2^{ n \choose 2 }$. Thus, we can expect that $\sH$ also grows quite fast. Therefore, $\var(\vt)$ won't converge as $n\to+\infty$. This may explain why \ERGMS fail in a principle way: its variance function simply doesn't converge! Therefore, it won't even stay as a natural exponential family distribution as the graph get bigger, leaving the Central Limit Theorem invalid. This may also explain why MCMC sampling for \ERGMS doesn't work well for large graphs: because of the violation of Central Limit Theorem.
\end{comment}

\subsection{Mass-Preserving ERGMs}
\begin{comment}
Even though a uniform continuous weighting function has desired properties as in Theorem \ref{thm:nef},
\end{comment}
%SK: need to define MPERGMs explicitly, however briefly!
To modify \ERGMS, we solve the optimization problem in
\eqref{eqn:nexp} with the uniform base measure $q\left(G\right)$ over
possible graphs $G\in\mathcal{G}_n$.  Let $t_a\left(G\right)=
K_{\vH}\left(\vt\left(G^\star\right); \vt\left(G\right)\right)$, a smoothed mass indicator in the neighborhood of the feature values for the observed graph.  The solution is an exponential family probability mass function
\begin{equation*}%\label{eqn:npergm}
\begin{split}
f\left(G\right) &= \frac{1}{Z\left(\vlambda,\lambda_a\right)}\exp\left[\inner{\vlambda}{\vt\left(G\right)} + \inner{\lambda_a}{t_a\left(G\right)}\right]\\
Z\left(\vlambda,\lambda_a\right) &= \sum_{G\in\mathcal{G}_n}\exp\left[\inner{\vlambda}{\vt\left(G\right)} + \inner{\lambda_a}{t_a\left(G\right)}\right].
\end{split}
\end{equation*}
which we refer to as {\em mass-preserving} ERGM (\NERGM).  The corresponding objective function
\begin{equation*}
l\left(\vlambda,\lambda_a\right) = \inner{\vlambda}{\vt\left(G^\star\right)} + \inner{\lambda_a}{t_a\left(G^\star\right)}-\ln{Z\left(\vlambda,\lambda_a\right)}-\beta\left|\lambda_a\right|
\end{equation*}
is concave.
\begin{comment}As described in Section \ref{sec:learning}, we use $t_a(G)=K\left(\left|\vt\left(G\right)-\vt\left(G^\star\right)\right|\right)$, for the single example graph $G^\star \in \Gn$, then we obtain \NERGM with an augmented parameter $\lambda_a$.
\end{comment}
%LYre
%Empirically, we find that uniform, Gaussian kernel functions work well for \NERGMS.
%To appeal to people working on \ERGMS for small graphs, we propose the use of uniform weighting functions. For large graphs, we propose the use of quadratic weighting functions. If necessary, other weighting functions in similar shape to kernel functions can be used to guarantee mass around the example graph. In Table \ref{tab:mif}, we listed 4 kinds of commonly used weighting functions. Applying these weighting functions with the algorithms described in section \ref{sec:learning}, we obtain \NERGM. We will demonstrate in Section \ref{sec:experiments} that \NERGM outperforms \ERGMS for the purpose of graph modeling and sampling.

\begin{comment}
For optimization, we use a slightly modified version of the Algorithm \ref{alg:coord}.
\end{comment}
There are several challenges with parameter estimation, most encountered before in \ERGM fitting \citep[e.g.,][]{Hunter2008a}.  As in the continuous case, the gradient cannot be computed in closed form except for graphs of small size ($\mathcal{G}_n$ for $n\leq 11$).  We therefore apply MCMC-MLE approach of \citet{Hunter2006}, computing $E_{f}\left[\vt\left(G\right)\right]$ in Algorithm \ref{alg:coord} as a sampled average $\frac{1}{S}\sum_{i=1}^{S}\left(\vt\left(G^i\right)\right)$ where $G^{1:S} \stackrel{i.i.d}{\sim}f\left(G|\vlambda,\lambda_a\right)$. There are, however, two complications with this approach.  One, graph sampling from \ERGMS is performed using Gibbs sampling and is computationally expensive.  Therefore, graphs $G^{1:S}$ are re-sampled only once in several iterations, and reused for other iterations with weights equal to the posterior probabilities.  Two, the resulting distribution over graphs can be multi-modal, and according to \citet{Jin2011,Hunter2006}, the sampler can get stuck around the closest mode leading to an incorrect estimate of the gradient.  Instead of performing line search, we use the direction of the gradient with a predefined step-size.

%% file: experiments.tex
\section{Experimental Evaluation}\label{sec:experiments}

\subsection{Non-Parametric Exponential Family Density Estimation}
We illustrate the behavior of the proposed non-parametric
density estimator matching first and second order moment constraints (NPGaussian, i.e. $\vt(x)=\left(x,x^2\right)$) in the univariate setting. Normal density (in $\EF$, $\mathcal{N}(0,1)$), mixture of two normals (not in $\EF$, $\frac{1}{2}\mathcal{N}(-3,1)+\frac{1}{2}\mathcal{N}(3,1)$), and a t-distribution (not in $\EF$, df=6) are used for simulating i.i.d samples. We vary the sample size from $10$ to $1000$ for training and compute the out-of-sample likelihood with an evaluation set of $100000$ samples for testing.
 We compared the performance of our non-parametric approach, the model from the true functional family, and another non-parametric approach (\KDE). There are two sets of tuning parameters, bandwidth $h$ and the box constraint parameter $\beta$, assumed to be the same for all $i=1,\dots,n$.  $\beta$ was set according to a fixed schedule
$\vec{\beta}(n)=O(1/\sqrt{n})$.  $h$ (both for KDE and for our approach) was determined based on cross-validated log-likelihood.\footnote{Similar to \KDE, the choice of kernel
width $h$ is important for obtaining good estimates. To
test how the non-parametric exponential family is affected by the
choice of $h$, we employed the same Gaussian kernel
function
to do density estimation with both \KDE and non-parametric
Gaussian. It appears that the best bandwidth are different.} % and then compare the performance of our method
%with the first and second moment constraints ($h_1\left(x\right)=x$
%and $h_2\left(x\right)=x^2$)
%to that of \KDE and the true model.
Gaussian kernel function is used for NPGaussian and for \KDE.
%Bandwidth parameters are selected by cross-validation for both \KDE and our non-parametric approach at different sample size $n$.
For estimating mixture distribution, the estimated NPGaussian model provides an approximation better than \KDE, and perhaps not surprisingly, better than GMM when the training sample size is small (Figure \ref{fig:lldenest}(a)).
For estimating normal density, the NPGaussian model
quickly converges to the normal density as suggested by Theorem
\ref{thm:true} (Figure \ref{fig:lldenest}(b)). We also consider the case when the true sufficient
statistics are given to us (constrained NPGaussian, CNPG). The CNPG model shows improvement over NPGaussian for small $n$.  However, as the training sample size increases, both CNPG and NPGaussian show similar performance as the moment constraints $\vt(x)$ are more accurately approximated.  We also experimented with $O(1/\log(n))$ regularization schedule for $\beta$s to estimate the mixed normal distribution.  As $n$ increase, the solution for NPGaussian is too sparse and gives a worse performance than \KDE(Figure \ref{fig:lldenest}(a)). However, it also enjoys a sparse set of augmented parameters $\vlambda_a$ (Figure \ref{fig:lldenest}(d)), whereas with schedule $O(1/\sqrt{n})$, NPGaussian keeps adding non-zero $\lambda_a$s.

%We generated $1.1n$ training points and do a 11-fold cross validation on these $1.1n$ samples.
%LY: removed same kernel width?
%LY: the following sentence also seems not covering t and normal distribution
%to approximate a bimodal
%mixture Gaussian distribution

\begin{figure*}[th]
\centering
\subfigure[]{\includegraphics[width=0.235\linewidth]{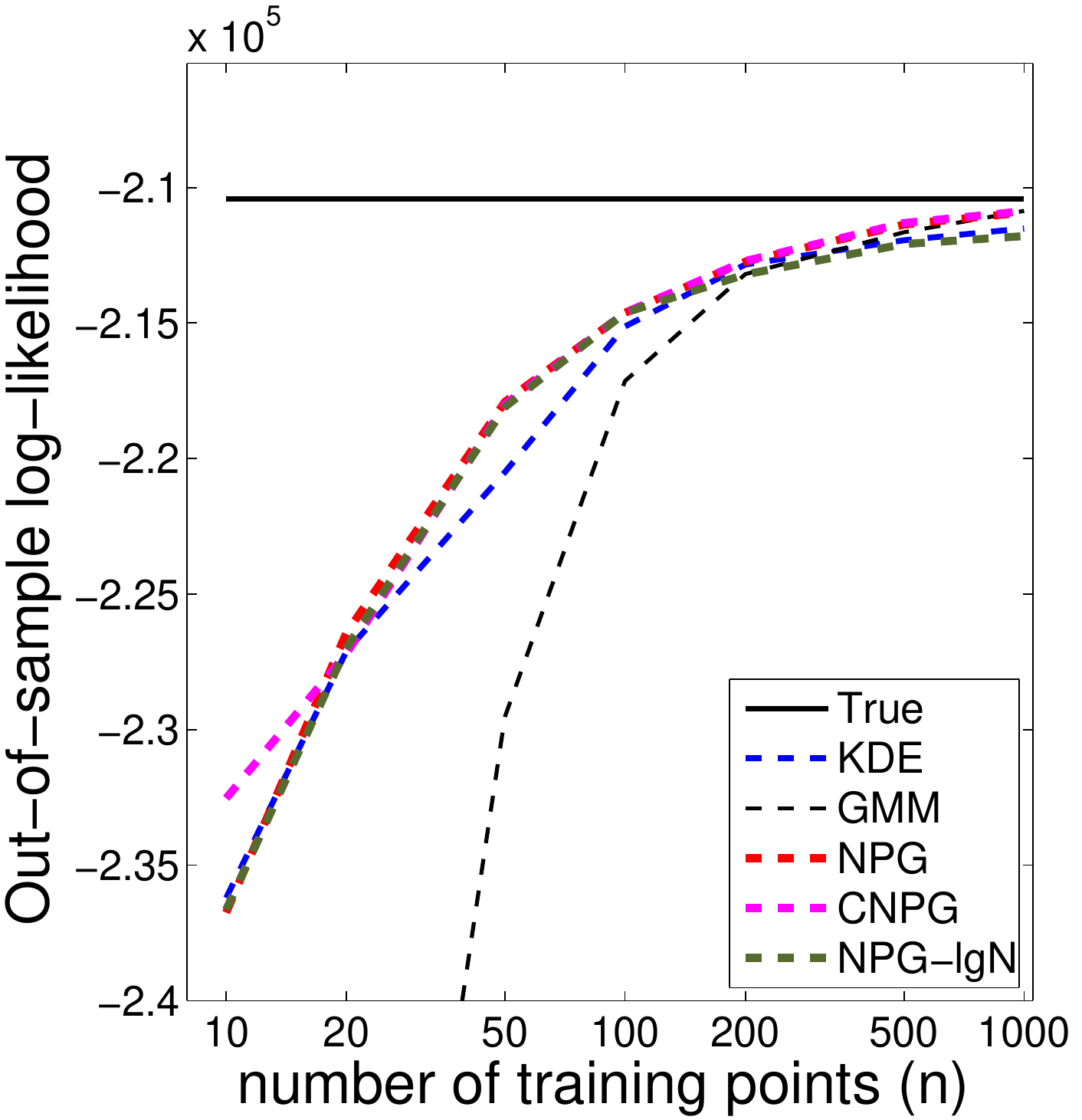}}
\subfigure[]{\includegraphics[width=0.235\linewidth]{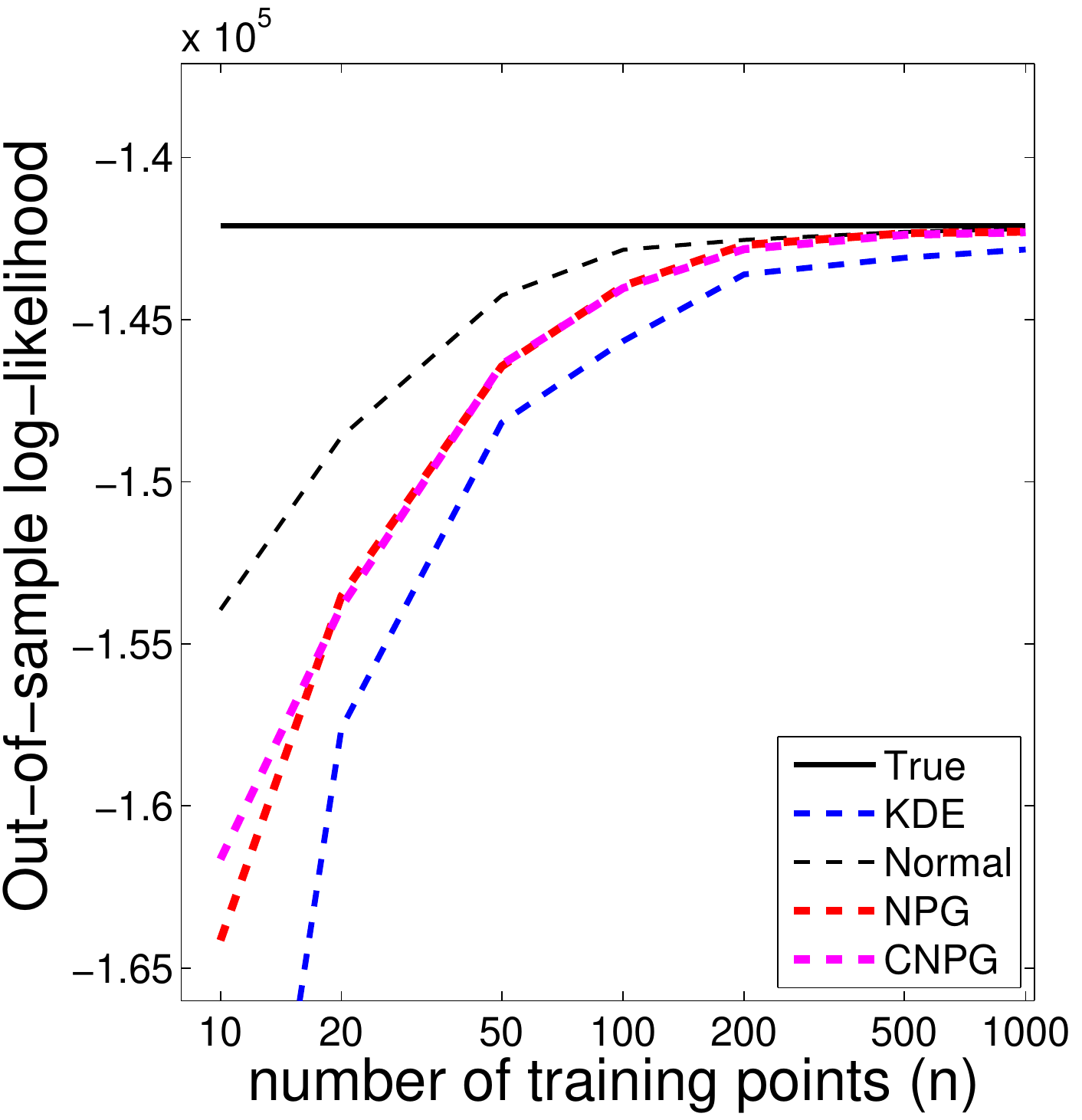}}
\subfigure[]{\includegraphics[width=0.235\linewidth]{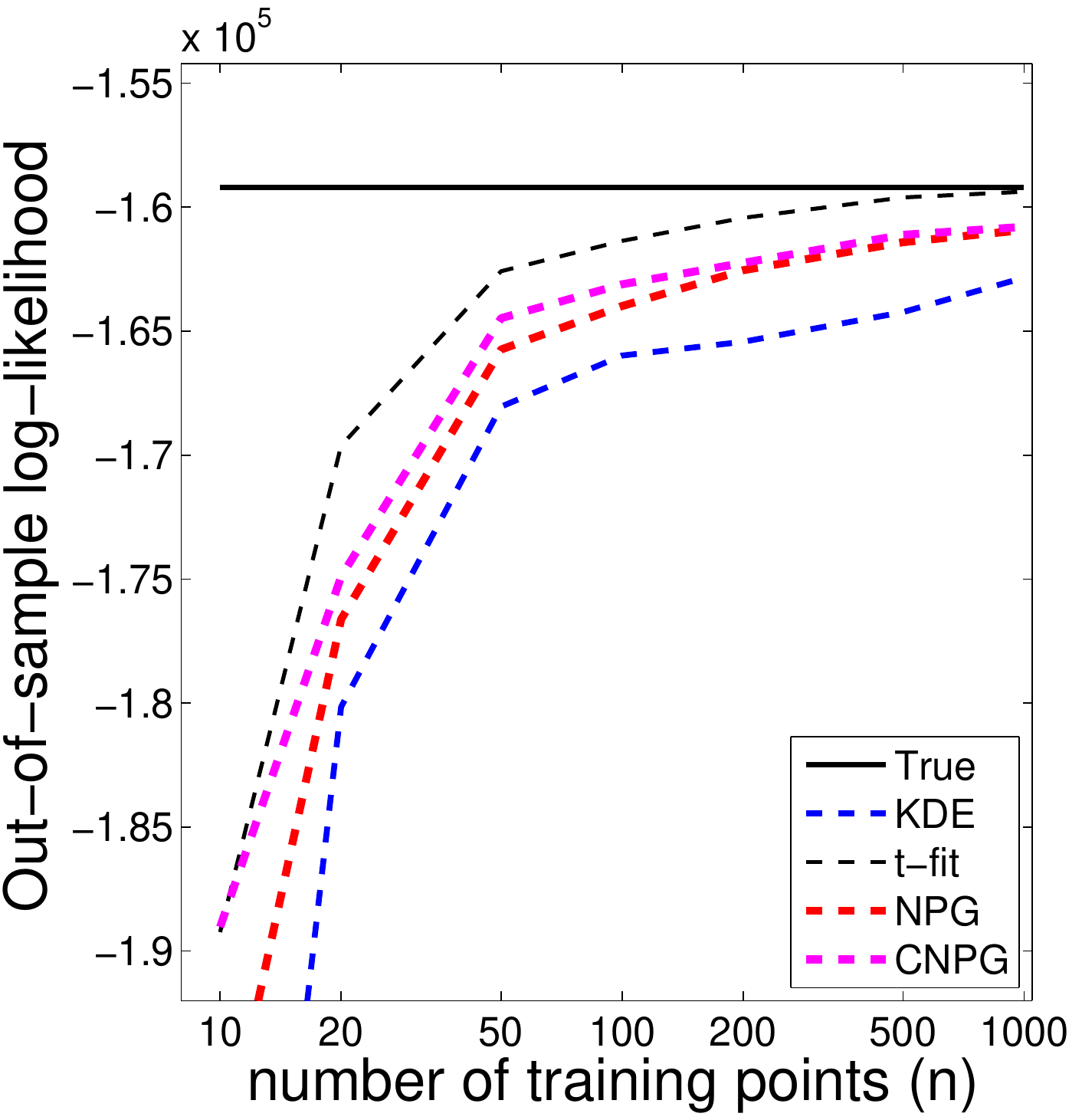}}
\subfigure[]{\includegraphics[width=0.235\linewidth]{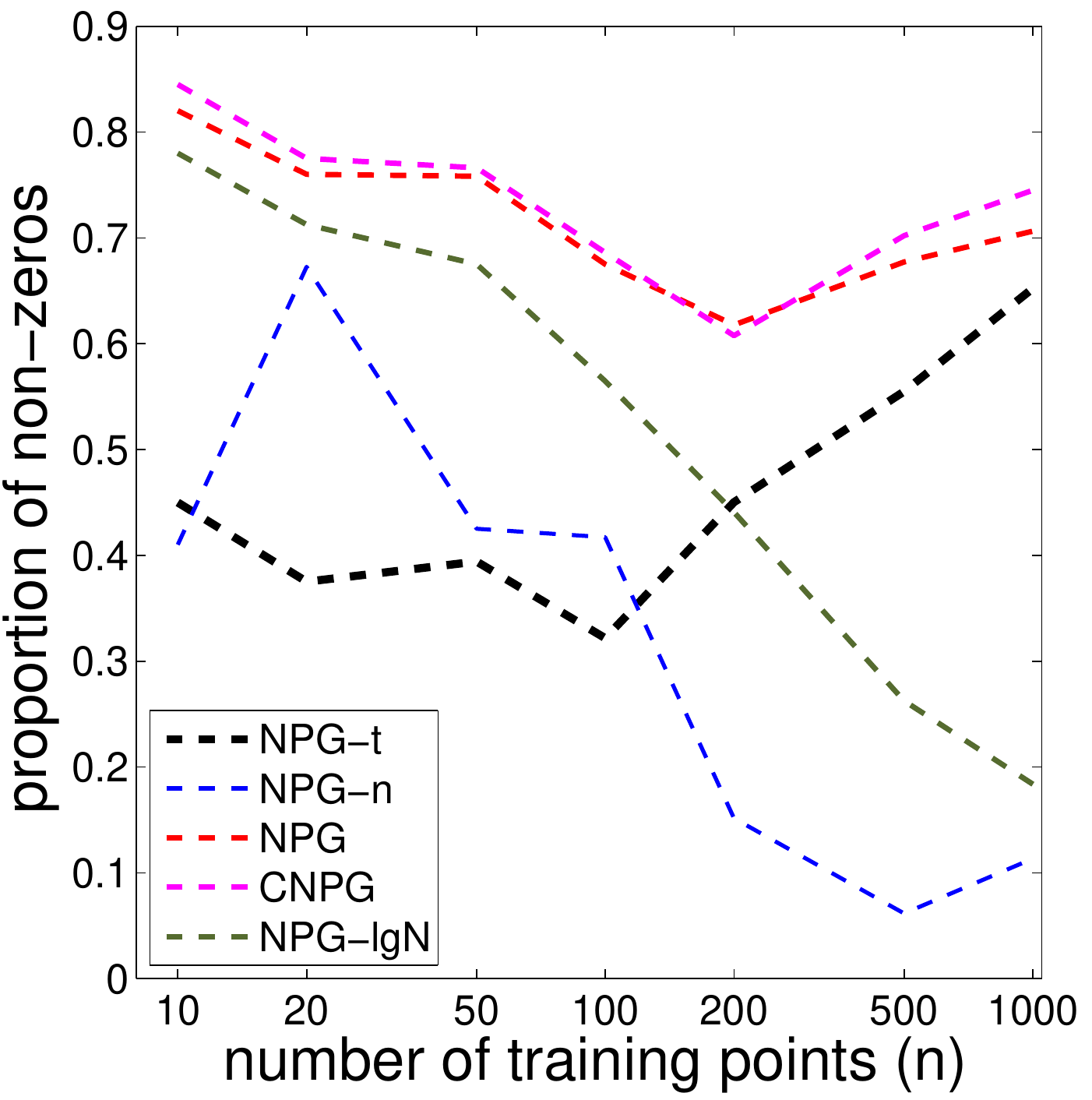}}
\caption{Estimating simple one dimensional densities. Results are
  averaged over 20 runs. The $x$ axis is in log scale. (a) Mixed normal distribution (b) Normal distribution (c) t distribution (d) Number of non-zero $\vlambda_a$s. In Figure(a,b,c), {\tt NPG}: NPGaussian with $O(1/\sqrt{n})$ schedule, {\tt NPG-lgN}: NPGaussian with $O(1/\log(n))$ schedule, {\tt CNPG}: constrained NPGaussian with true global moment statistics. In (d), {\tt NPG-t}: number of non-zeros for estimating t distribution with NPG, {\tt NPG-n}: number of non-zeros for estimating normal distribution with NPG, {\tt NPG}: number of non-zeros for estimating mixed normal distribution with NPG, {\tt CNPG}: number of non-zeros for estimating mixed normal distribution with CNPG, {\tt NPG-lgN}: number of non-zeros for estimating mixed normal distribution with NPG-lgN. }\label{fig:lldenest}
\end{figure*}

\subsection{Modeling Graphs with \NERGMS}
\begin{comment}
We follow the common ways in network generation studies, to
\end{comment}
We evaluate the fit of the estimated models by comparing local statistics of the observed graph to that of the samples generated from the estimated distribution.\footnote{See \cite{Hunter2008a} for a discussion on the evaluation of fit for social networks.}

\begin{comment}
The statistics of the observations not falling within the range of those simulated from the model
%\begin{comment}   a set
of observed statistics with a range of the same statistics obtained by
simulating many networks from the fitted model. If the observed
network is not typical of the simulated network for a particular
statistic,
would be indicative of the estimated model being is either degenerate or simply a misfit.
\end{comment}

We
make use of three sets of local statistics commonly used as goodness-of-fit
measures for \ERGMS %These measures are important indicators of how realistic network structures form in general
\cite{Hunter2008a}: the {\em degree distribution} (the proportion of nodes with exactly $k$ neighbors),
\begin{comment}
which is defined as the statistics $D_{0},D_{1},\cdots,D_{n-1}$, with
each $D_{i}$ representing the number of nodes with $i$ edges connected
to them divided by $n$.
\end{comment}
{\em edgewise shared partner distribution} (the proportion of edges joining nodes with exactly $k$ neighbors in common),
\begin{comment}
which is defined as the statistics
$EP_{0},EP_{1},\cdots,EP_{n-2}$, with each $EP_{i}$ representing the
number of edges in the graph between two nodes that share exactly $i$
neighbors in common divided by the total number of edges. We also use
\end{comment}
and the {\em minimum geodesic distance} (the proportion of connected node-pairs which has a minimum distance of $k$).

\begin{table*}[!tb]
\centering
\caption{Social network data sets. {\tt g8}: The 8-node graph as in Figure \ref{fig:g8degen}(a); {\tt Do}: The dolphins data set \citep{Lusseau2003}; {\tt Kp}: The Kapferer data set \citep{Hunter2008a}; {\tt Fl}: The Florentine Business data set \citep{Hunter2008a}; {\tt Fa}: The Faux.Mesa.High data set \citep{Hunter2008a}; {\tt Ja}: The Jazz data set \citep{Gleiser2003}; {\tt Ad}: The AddHealth data set \citep{Harris2008}; {\tt Fb}: The Facebook data set \citep{Moreno2009}; {\tt Em}: The Email data set \citep{Guimera2003}.}   \label{tab:sndata}
\begin{tabular}{|c|c|c|c|c|c|c|c|c|c|}
\hline
& {\tt g8} & {\tt Do} & {\tt Kp} & {\tt Fl} & {\tt Fa} & {\tt Ja} & {\tt Ad} & {\tt Fb} & {\tt Em} \\
\hline
$|V|$ & 8 & 62 & 39 & 16 & 206 & 198 & 803 & 1024 & 1133\\
\hline
$t_e(G^\star)$ & 22 & 159 & 158 & 15 & 203 & 2742 & 1985 & 1012 & 5451\\
\hline
$t_{\triangle}(G^\star)$ & 29 & 95 & 201 & 5 & 62 & 17899 & 649 & 116 & 5343\\
\hline
No. unique sampled graph & 99 & 100 & 100 & 100 & 100 & 100 & 100 & 96 & 100\\
%No. unique sampled graph & 99 & 100 & 100 & 100 & 100 & 100 & 100 & 96 & 100\\
\hline
No. unique features & 33 & 33 & 30 & 27 & 14 & 70 & 72 & 74 & 70 \\
%No. unique features & 36 & 100 & 30 & 27 & 14 & 70 & 72 & 74 & 70 \\
\hline
No. max hop of samples & 22 & 306 & 263 & 39 & 341 & 1435 & 999 & 385 & 1275\\
%No. max hop of samples & 26 & 1238 & 263 & 39 & 341 & 1435 & 999 & 385 & 1275\\
\hline
\end{tabular}
\end{table*}

\begin{figure*}[!tb]
\centering
\includegraphics[width=0.98\linewidth]{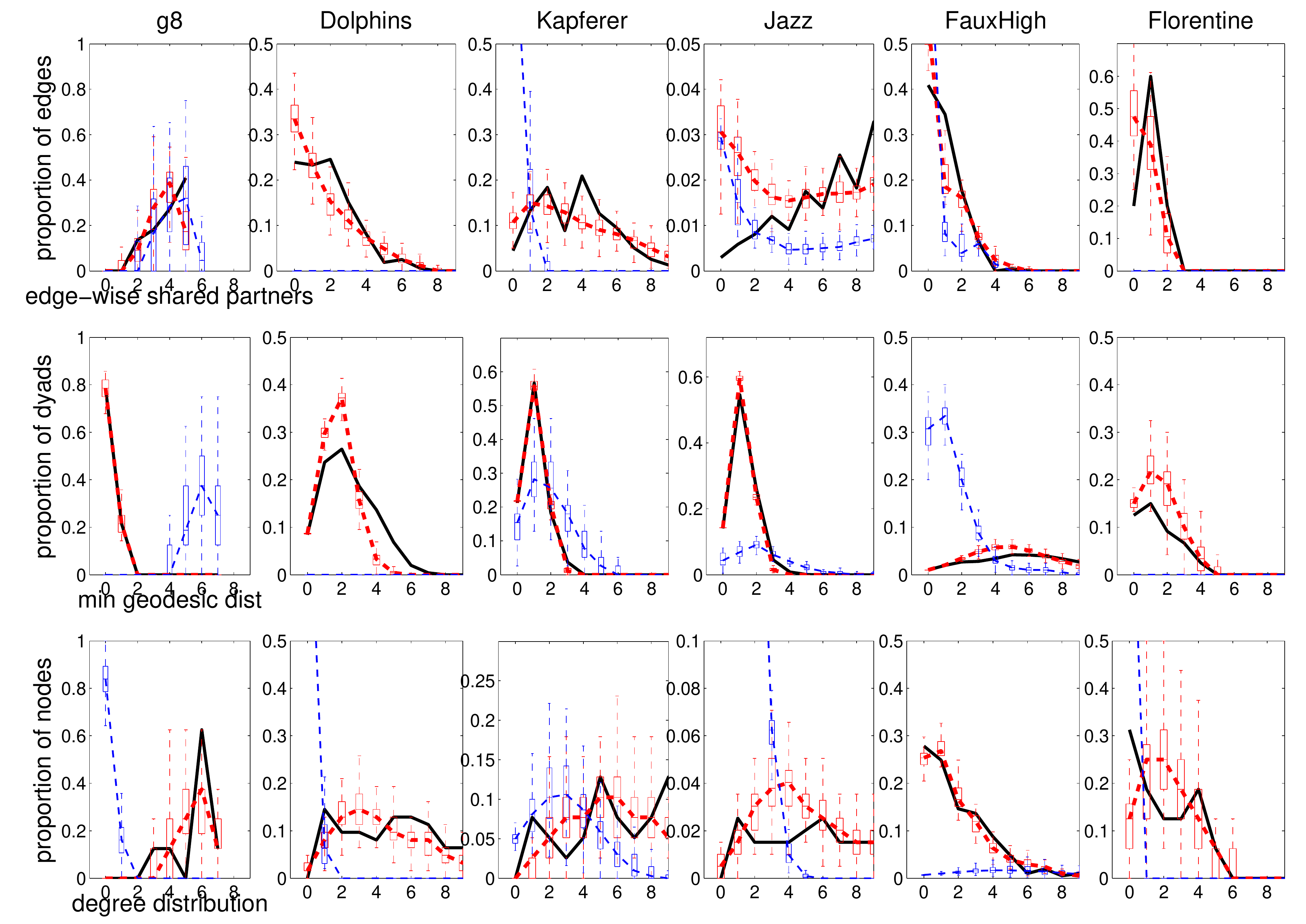}
\caption{Goodness of fit for small graphs. Gaussian kernel
  functions are used for \NERGM. \ERGM is shown in blue dashed lines, and \NERGM is shown in red dashed lines. Black lines are the statistics for $G^\star$, being closer to black line means better fit.}\label{fig:gof}
\end{figure*}

\begin{figure*}[!h]
\centering
\includegraphics[width=0.98\linewidth]{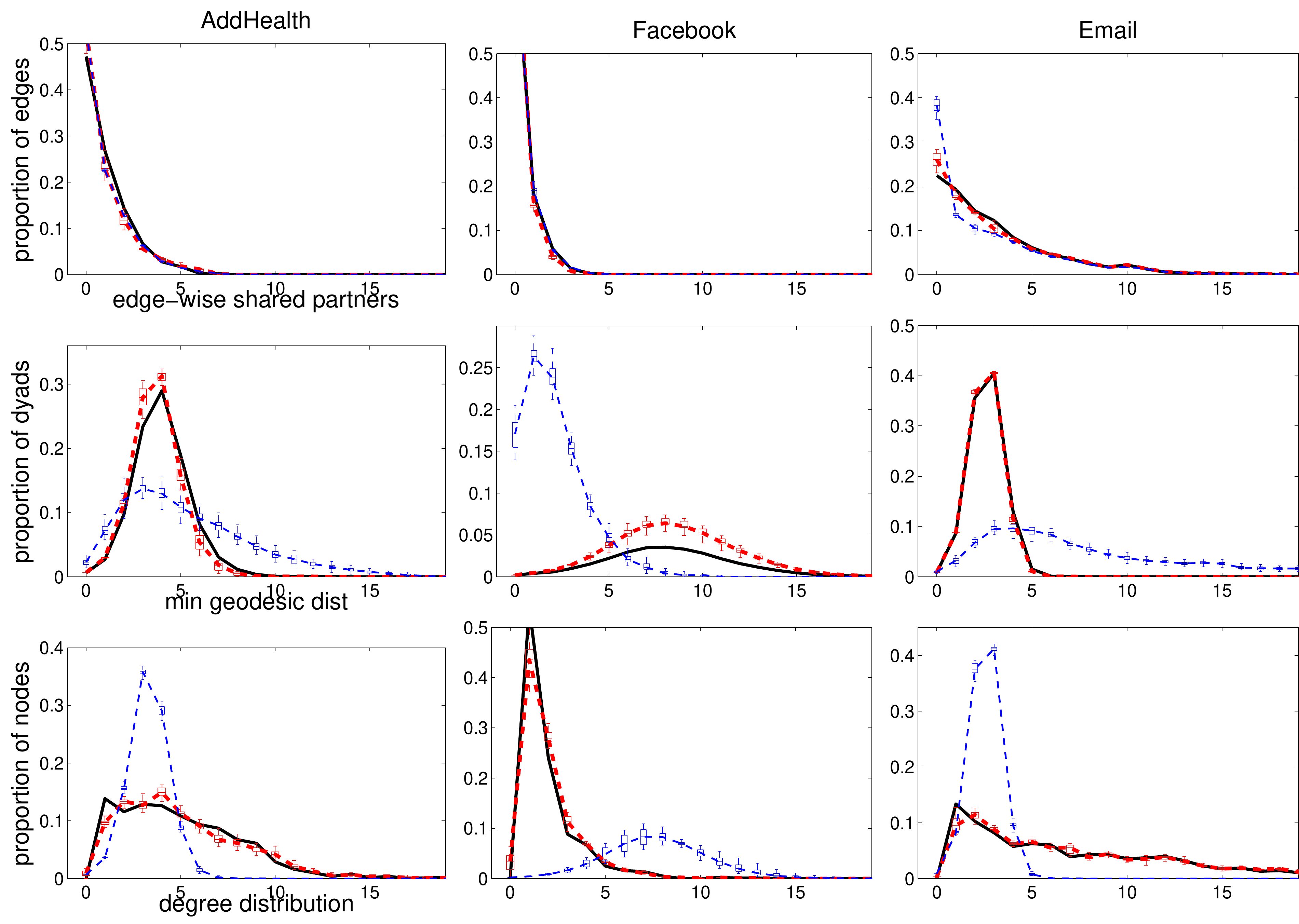}
\caption{Goodness of fit for large graphs. Gaussian kernel
  functions are used for \NERGM. \ERGM is shown in blue dashed lines, and \NERGM is shown in red dashed lines.}\label{fig:gof_thou}
\end{figure*}

\begin{comment}
$MGD_{1},MGD_{2},\cdots,MGD_{n-1}$, with $MGD_{i}$ representing the number of node-pairs that has $i$ as their shortest path length, divided by the total number of connected node-pairs.
\end{comment}

%We learn our \NERGM with regular uniform weighting functions for a
%small synthetic graph $\Gnk{8}$. As no quadratic weighting
%functions are used, we are keeping the original \ERGM distribution as
%much as possible while also fixing the degeneracy.  For larger graphs, we adopted both the Gaussian weighting function and the quadratic weighting function. Empirically, quadratic weighting functions worked better in obtaining non-degenerate models that also preserve variability.

%SK: Create a table with the following columns: name of the data set, number of nodes, number of edges, number of triangles, citation.  Move all of the relevant info there.

We consider the number of edges and triangles as sufficient statistics, $\vt\left(G\right)=\left(t_{e}\left(G\right), t_{\triangle}\left(G\right)\right)$.  First, we consider the toy domain of graphs with $8$ nodes,
$\Gnk{8}$.  We enumerate all possible $K=12346$ non-isomorphic graphs and resulting feature tuples, and compute probability mass entries $\pi_1,\dots,\pi_K$. We trained our \NERGM with
a Gaussian kernel function with $h=8, \beta=0.2$. Figure \ref{fig:g8degen} shows that \NERGM puts larger probability mass around $G^\star$.  %{\bf NEED A PLOT WITH THE STATISTICS WITH WHICH WE ARE EVALUATING THE REST OF THE GRAPHS}
\begin{comment}
%SK: should be moved into the table
Take the graph with 22 edges and 29 triangles for example, \ERGM learned a degenerate model that will generate close to complete graph most of the time, whereas \NERGM learned a model with the mode close to the example graph. (Figure \ref{fig:g8degen}) (supplemental summary material, table of all 228 feature pairs and its corresponding degeneracy summary)?
\end{comment}

%{\bf modify this junk paragraph, 1. write about the datasets 2. write about NERGM parameters 3. summarize results}

We also estimated \NERGMS for several social network data sets,
ranging in the number of nodes from $16$ to $1024$, and with varying
density of edges.  Since the number of nodes $n$ for these graphs are
too large to enumerate $\mathcal{G}_n$, the graphs are drawn using
Gibbs sampler, and the parameters for \NERGMS (and \ERGMS, using the R package {\tt ergm}
\citep{Hunter2008a}) are estimated using MCMC-MLE. Then 100 samples
were generated using the Markov Chain with learned parameters. For
\NERGM, the Markov chain was initialized with the example graph,
whereas we are not sure what initial state was used by {\tt ergm}. We
then run the chain for a burn-in of 1000 iterations and then use 100
iterations between each draw. The 100 samples are then used to plot
the graph statistics for goodness-of-fit test in Figure \ref{fig:gof}
and Figure \ref{fig:gof_thou}. For each estimated model, the statistics in Figure
%\ref{fig:gof} (large graphs not shown)
\ref{fig:gof} \& \ref{fig:gof_thou}
were generated from 100 sampled graphs obtained by running Gibbs with
1000 iterations for burn-in and 100 iterations between samples. We initialized our Markov chain with the example graph, whereas we are not sure what initial state did {\tt ergm} use. We used a set of
hand-tuned step-size and $h$ for different data sets, and
re-scaled the edge and triangle features by a factor of
$\frac{1}{t_e(G^\star)}$ and
$\frac{1}{t_{\triangle}(G^\star)}$. Empirically, we find
$h\approx 8$ and a predefined step-size $10$ works well for small
graphs. For graphs with several thousands of nodes, the pre-defined step-size and
$h$ needs to be larger to guarantee reasonable variance and
concentration of the model. In Figure \ref{fig:gof}, \ERGM is degenerate for the {\tt Florentine} and {\tt Dolphins} dataset, because most sampled graphs have 0-degree nodes (third row), while \NERGM is able to generate samples scoring a similar set of graph statistics. In order to investigate the variance of the learned \NERGM, we count the number of samples that are different in structure (not counting isomorphism) or different in features (number of edges and triangles), while recording the maximum number of unique edge-flips needed to get from the initial state to the sampled graph (number of max hops). The results in Table \ref{tab:sndata} suggests that our sampler explores $\Gn$ with a considerable range.

%% file: conclusion.tex
\section{Discussion and Conclusions}\label{sec:conclusion}

We have proposed a non-parametric exponential family model that is capable of approximating distributions that do not fall within the exponential family empirically.  As the data size increases, this model can approximate arbitrary (continuous) densities with tuning parameters controlling the sparsity. And if the true density falls within the exponential family characterized by the
chosen features,
%chosen sufficient statistics,
the estimated non-parametric model converges to the parametric one.  The proposed framework results in a non-parametric density estimator which admits global constraints; if available, this information may require fewer data points to approximate the underlying density.  The resulting MLE optimization problem is concave with $\ell_1$ penalty term, but raises a computational challenge because the number of inequality constraints is proportional to the number of data points. We also adapted the approach to modify exponential random graph models for graphs to come up with an exponential family model averse to model degeneracy.

As future directions, we would like to investigate the rules for selecting bandwidth parameters, the acceleration of the optimization problem, and efficient sampling techniques for sampling from the non-parametric exponential family.

%% file: appendix_formal.tex
\section{Proofs}\label{sec:proofs}
\subsection{Theorem \ref{thm:true}}\label{sec:proofs:thm1}
To prove Theorem \ref{thm:true}, let's take a closer look at the
augmented constraints. First consider the augmented
statistics $\vt_a^{i,\star}$ for sample $\vx^i$.

\begin{lemma}\label{lemma:kde}
$E_{\hat{f}_n\left(\vx\vert\vx^{1:n}\right)}\left[t_a^i\left(\vx\right)\right]=f_n^{KDE}\left(\vx^i \vert \vx^{1:n}\right)$
\end{lemma}
\begin{proof}
\begin{align}
E_{\hat{f}_n\left(\vx\vert\vx^{1:n}\right)}\left[t_a^i\left(\vx\right)\right]
&=
\frac{1}{n}\sum\limits_{j=1}^{n}K_{\vH}\left(\vx^i;\vx^j\right)
= f_n^{KDE}\left(\vx^i \vert \vx^{1:n}\right),\nonumber
\end{align}
\end{proof}
i.e., the augmented statistics are the \KDE estimates for samples
$\vx^1 \dots \vx^n$ respectively. \citet{Parzen1962} proved the
pointwise convergence of \KDES and \citet{Nadaraya1965} proved the
uniform convergence for \KDES under further assumptions. We include
the first half of \citep[][Theorem 1]{Nadaraya1965} below since it is
essential for proving both Theorem \ref{thm:true} and Theorem
\ref{thm:nef}.

\begin{theorem}\label{thm:ekde}
Let  $\overline{f}_n^{KDE}\left(\vx\right) =
E_{f\left(\vy^{1:n}\right)}\left[f_n^{KDE}\left(\vx\vert\vy^{1:n}\right)\right]$
be the expected value of the \KDE density given sample points
$\vy^{1:n}\iid f$.

Suppose $K_{\vH}\left(\vx\right): \vx \in \sX \to \RR$ is a function of bounded variation
and $f\left(\vx\right)$ is a uniformly continuous density function,
and the series $\sum_{n=1}^{\infty}e^{-\gamma n \left|\vH\right|}$
converges for every positive value of $\gamma$.

Then $f_n^{KDE}\left(\vx \vert \vx^{1:n}\right)
\stackrel{a.s.}{\to} \overline{f}_n^{KDE}\left(\vx\right)$ uniformly
on $\sX$.

That is, $\sup_{\vx \in \sX} \left|f_n^{KDE}\left(\vx\right)-\overline{f}_n^{KDE}\left(\vx\right)\right|
\stackrel{a.s.}{\to} 0$.
\end{theorem}

\begin{remark}\label{remark:ekde}
It is helpful to note that given $\vx^1 \dots \vx^n \iid f$,
\begin{align}
\overline{f}_n^{KDE}\left(\vx^i\right)
&=\int\limits_{\vx^1\dots\vx^n} \prod\limits_{j=1}^{n}f(\vx^j)\frac{1}{n}
\sum\limits_{j=1}^{n} K_{\vH}\left(\vx^i ; \vx^j\right) \diff {\vx^1
\dots \vx^n}\nonumber\\
&=\frac{1}{n}\sum\limits_{j=1}^{n}\int\limits_{\vx^j\in\sX}K_{\vH}\left(\vx^i;
\vx^j\right) f\left(\vx^j\right)\diff {\vx^j} = \int\limits_{\vx\in\sX}K_{\vH}\left(\vx^i ; \vx\right) f\left(\vx\right)\diff
   {\vx} = E_f\left[t_a^i\left(\vx\right)\right].\nonumber
\end{align}
\end{remark}

\noindent{\bf Theorem \ref{thm:true}}
For a density $f\left(\vx\right)\in\EF_{\vt}$, with domain space $\sX$, feature space $\sH$, and canonical parameters $\vlambda \in \sC\in\RR^d$.
Suppose $\vx^1,\dots,\vx^n \triangleq \vx^{1:n}$ is a sequence
of independent and identical random vectors drawn from $f$.  Let
$f_n^{NE}\left(\vx\vert \hat{\vtheta}_n, \vx^{1:n} \right)\in\mathcal{NEF}_{\vs}$
be the MLE solution of \eqref{eqn:nexp} $\hat{\vtheta}_n =
\left(\tilde{\vlambda}_n, \tilde{\vlambda}_{a,n}\right)$, with all $\beta_i=\beta>0$, $i=1,\dots,n$. Assuming
\begin{enumerate}
\setlength{\itemsep}{0em}
\setlength{\parskip}{0pt}
\setlength{\parsep}{0pt}
\item $\sX$ is compact,
\item %$\vt\left(\vx\right)$
$\vt$ is continuous,
\item $\EF_{\vt}$ is a family of uniformly equicontinuous functions w.r.t $\vx$,
%\item In \eqref{eqn:nexp}, we choose $\beta_i = \beta > 0$, $\forall i=1,\dots,n$.
\item Kernel $K$ has bounded variation and has a
  bandwidth parameter $\vH$ such that the series $\sum_{n=1}^{\infty}e^{-\gamma n \left|\vH\right|}$
converges for every positive value of $\gamma$,
\end{enumerate}
then as $n\to\infty$, $\tilde{\lambda}_{a,n}^i\stackrel{p}{\to}0,
\forall i=1,\dots,n$ and $\tilde{\vlambda}_n\stackrel{p}{\to}\vlambda$.

\begin{proof}
Consider the estimated densities
$f_n^E\left(\vx\vert\hat{\vlambda}_n\right)\in\mathcal{EF}_{\vt}$ and
$f_n^{NE}\left(\vx\vert\hat{\vtheta}_n, \vx^{1:n} \right)\in\mathcal{NEF}_{\vs}$.

MLE for regular exponential families converges in probability to the true values of
parameters, $\hat{\vlambda}_n\stackrel{p}{\to}\vlambda$. \citep[e.g.][Chapter 6]{Wainwright2008},\citep[e.g.][Chapter 5.2]{Vaart1998}

Since $f\left(\vx | \vlambda\right)$ is continuous w.r.t $\vlambda$, given any $\vx \in \sX$, $f^{E}_{n}\left(\vx|\hat{\vlambda}_n\right) \stackrel{p}{\to} f\left(\vx|\vlambda\right)$.
Because $\EF_{\vt}$ is an equicontinuous family, and $\sX$ is compact, according to \citep[][Problem 9.6]{Royden1988}, $f_n^{E} \stackrel{p}{\to}f$ uniformly on $\sX$.

Thus given any $\delta, \beta > 0$, there exists $N$ s.t. when $n>N$, $\forall \vx \in \sX$, $P(\left|f_n^E\left(\vx\vert\hat{\vlambda}_n\right)-f\left(\vx\vert\vlambda \right)\right|>\frac{\beta}{2})<\delta$. Then, we have
\begin{align}
P\left(\left|E_{f_n^E\left(\vx\vert\hat{\vlambda}_n\right)}\left[t_a^i\left(\vx\right)\right] - E_{f}\left[t_a^i\left(\vx\right)\right]\right|>\frac{\beta}{2}\right)
&= P\left(\int_{\vx\in\sX}\left|f_n^E\left(\vx\vert\hat{\vlambda}_n\right)-f\left(\vx\vert\vlambda \right)\right|t_a^i\left(\vx\right) \diff \vx > \frac{\beta}{2}\right)\nonumber\\
&\leq P\left(\int_{\vx\in\sX}\sup_{\vx\in\sX}\left|f_n^E\left(\vx\vert\hat{\vlambda}_n\right)-f\left(\vx\vert\vlambda \right)\right|t_a^i\left(\vx\right) \diff \vx > \frac{\beta}{2}\right)\nonumber\\
&= P\left(\sup_{\vx\in\sX}\left|f_n^E\left(\vx\vert\hat{\vlambda}_n\right)-f\left(\vx\vert\vlambda \right)\right|>\frac{\beta}{2}\right)\nonumber\\
&< \delta.\nonumber
\end{align}

Thus for all $\vx^i \in \sX$, given any $\delta_1, \frac{\beta}{2}>0$, there exist
$N_1$, such that when $n>N_1$,
\begin{equation*}
P\left(\left|E_{f_n^E\left(\vx\vert\hat{\vlambda}_n\right)}\left[t_a^i\left(\vx\right)\right]-E_{f}\left[t_a^i\left(\vx\right)\right]\right|>\frac{\beta}{2}\right)<\delta_1.
\end{equation*}
%uniformly.

Meanwhile, since we have
$E_{\hat{f}_n\left(\vx\vert\vx^{1:n}\right)}\left[t_a^i\left(\vx\right)\right]$
for the box constraints in \eqref{eqn:nexp}, we would like to have 
\begin{equation*}
E_{\hat{f}_n\left(\vx\vert\vx^{1:n}\right)}\left[t_a^i\left(\vx\right)\right]
{\to} E_{f}\left[t_a^i\left(\vx\right)\right]
\end{equation*}
for any $\vx^i \in \sX$ uniformly as well. This is satisfied under assumptions of Theorem
\ref{thm:ekde}. Then we would have
\begin{equation*}
E_{\hat{f}_n\left(\vx\vert\vx^{1:n}\right)}\left[t_a^i\left(\vx\right)\right]\stackrel{a.s.}{\to}E_{f}\left[t_a^i\left(\vx\right)\right]
\end{equation*}
for all $i=1 \dots n$. Because almost sure convergence implies
convergence in probability, we have: for all $\vx^i \in \sX$, given
any $\delta_2, \frac{\beta}{2}>0$, there exists $N_2$, s.t. when
$n>N_2$, 
\begin{equation*}
P\left(\left|E_{\hat{f}_n\left(\vx\vert\vx^{1:n}\right)}\left[t_a^i\left(\vx\right)\right]-E_{f}\left[t_a^i\left(\vx\right)\right]\right|>\frac{\beta}{2}\right)<\delta_2.
\end{equation*}
%uniformly.

Thus using Triangle Inequality and Boole Inequality, when $n>N=\max(N_1,N_2)$,
\begin{equation}
P\left( \left|E_{\hat{f}_n\left(\vx\vert\vx^{1:n}\right)}\left[t_a^i\left(\vx\right)\right] - E_{f_n^E\left(\vx\vert\hat{\vlambda}_n\right)}\left[t_a^i\left(\vx\right)\right]\right|> \beta\right) \leq \delta_1 + \delta_2.\nonumber
\end{equation}

This means that the additional constraints will be matched by $[\hat{\vlambda},\vec{0}]$ in probability. Because $\hat{\vlambda}_n$ is the solution to \eqref{eqn:moments} without the additional
constraints, thus it will have smaller KL divergence (larger entropy)
than $\hat{\vtheta}_n$. Therefore $[\hat{\vlambda}_n,\vec{0}]$ is
bound to be the MLE solution to \eqref{eqn:nexp}. Since
$\hat{\vlambda}_n \stackrel{p}{\to} \vlambda$, we have
$\tilde{\lambda}_{a,n}^i\stackrel{p}{\to}0, \forall i=1,\dots,n$
and $\tilde{\vlambda}_n\stackrel{p}{\to}\vlambda$.
\end{proof}

\begin{remark}
It can be noted that a sufficient condition for $\EF_{\vt}$ to be an equicontinuous exponential
family is that $\sX$ and $\sC$ are both compact. Since in reality, we
rarely deal with probability density functions with infinite density
values, the conditions for Theorem \ref{thm:true}, though look
restrictive, do not constrain the application domain much. However,
since for any regular exponential family the canonical parameter space is open
\citep[][Theorem 3.6]{Brown1986}, this means Theorem \ref{thm:true} for non-parametric
exponential family works only on a closed subset $\sC$, of the original
canonical parameter space.
\end{remark}

\subsection{Theorem \ref{thm:nef}}\label{sec:proofs:thm2}

To prove the pointwise convergence result for our non-parametric
exponential family, we rely on the convergence of \KDES and Triangle
Inequalities. We start off by introducing several lemmas.

\begin{lemma}\label{lemma:enekde}
Let  $\tilde{f}_n^{KDE}\left(\vx\right) =
E_{f_n^{NE}\left(\vy^{1:n}\vert\hat{\vtheta}_n,\vx^{1:n}\right)}\left[f_n^{KDE}\left(\vx\vert\vy^{1:n}\right)\right]$
be the expected value of the \KDE density given sample points
$\vy^{1:n}\iid f_n^{NE}\left(\vx\vert\hat{\vtheta}_n,\vx^{1:n}\right)$. Then $\tilde{f}_n^{KDE}\left(\vx^i\right) = E_{f_n^{NE}\left(\vx\vert\hat{\vtheta}_n,\vx^{1:n}\right)}\left[t_a^i\left(\vx\right)\right]$.
\end{lemma}

\begin{proof}
Similar as Remark \ref{remark:ekde},

\begin{align}
\tilde{f}_n^{KDE}\left(\vx^i\right)
&=\int\limits_{\vx^1\dots\vx^n} \prod\limits_{j=1}^{n}f_n^{NE}\left(\vx^j\right)\frac{1}{n}
\sum\limits_{j=1}^{n} K_{\vH}\left(\vx^i ; \vx^j\right) \diff {\vx^1\dots\vx^n}\nonumber\\
&= \frac{1}{n}\sum\limits_{j=1}^{n}\int\limits_{\vx^j\in\sX}K_{\vH}\left(\vx^i;
\vx^j\right) f_n^{NE}\left(\vx^j\right)\diff {\vx^j} =\int\limits_{\vx\in\sX}K_{\vH}\left(\vx^i ; \vx\right) f_n^{NE}\left(\vx\vert\right)\diff
   {\vx}\nonumber\\
&=E_{f_n^{NE}\left(\vx\right)}\left[t_a^i\left(\vx\right)\right].\nonumber
\end{align}
\end{proof}

A pointwise convergence result for the expected \KDE density can be found in \cite{Parzen1962}:

\begin{theorem}\label{thm:weakekde}
Suppose $f\left(\vx\right)$ is any probability density function which is continuous at point $\vx^0$.
Let  $\overline{f}_n^{KDE}\left(\vx\right) =
E_{f\left(\vy^{1:n}\right)}\left[f_n^{KDE}\left(\vx\vert\vy^{1:n}\right)\right]$
be the expected value of the \KDE density given sample points
$\vy^{1:n}\iid f$.

Suppose $K_{\vH}\left(\vx\right): \vx \in \sX \subset \RR^m \to \RR$ is a
probability density function which satisfies:

\begin{align}
\sup_{\vx \in \sX} K_{\vH}\left(\vx\right) &< \infty,\label{eqn:KK1}\\
\lim\limits_{\normof{\vx}\to\infty}
K_{\vH}\left(\vx\right)\prod\limits_{i=1}^{m}\vx_i &= 0,\label{eqn:KK2}\\
\lim\limits_{n\to\infty}\left|\vH\right|^{\frac{1}{2}} &= 0. \label{eqn:KK3}
\end{align}

Then $\lim\limits_{n\to\infty}\overline{f}_n^{KDE}\left(\vx^0\right)=f\left(\vx^0\right)$.

That is, the expected \KDE density at continuity point $\vx^0$
converges to the sampling probability density function $f$ at $\vx^0$.
\end{theorem}

\begin{corollary}\label{coro:convnekde}
Given a sample $\vx^i$, $\lim\limits_{n\to\infty}E_{f_n^{NE}\left(\vx\right)}\left[t_a^i\left(\vx\right)\right]=f_n^{NE}\left(\vx^i\right)$ if $K_{\vH}\left(\vx\right)$ is continuous and \eqref{eqn:KK1},\eqref{eqn:KK2},\eqref{eqn:KK3} are satisfied.
\end{corollary}
\begin{proof}
If $K_{\vH}\left(\vx\right)$ is continuous, then $f_n^{NE}$ satisfies Theorem \ref{thm:weakekde}'s conditions for $f$. Combining Lemma \ref{lemma:enekde}, we have then have Corollary \ref{coro:convnekde}.
\end{proof}

In addition, we have the mean-square convergence of the \KDE density
stated in \cite{Parzen1962} as well:

\begin{theorem}\label{thm:realkde}

Suppose $K_{\vH}\left(\vx\right): \vx \in \sX \subset \RR^m \to \RR$ is a
probability density function which in addition to
\eqref{eqn:KK1},\eqref{eqn:KK2},\eqref{eqn:KK3}, also satisfies:

\begin{equation} \label{eqn:KK4}
\lim\limits_{n\to\infty}n\left|\vH\right|^{\frac{1}{2}} = \infty
\end{equation}

and $f\left(\vx\right)$ is probability density function which is continuous at point $\vx^0$.

Then
$\lim\limits_{n\to\infty}E_{f\left(\vx^{1:n}\right)}\left[\left(f_n^{KDE}\left(\vx^0\vert
  \vx^{1:n}\right)-f\left(\vx^0\right)\right)^2\right]=0$.

That is, the \KDE density at continuity point $\vx^0$
converges in mean square to the true density $f$ at $\vx^0$.
\end{theorem}

\begin{lemma}\label{lemma:convatsinglepoint}
\begin{equation*}
f_n^{NE}\left(\vx^i\right) \stackrel{p}{\to} f\left(\vx^i\right)
\end{equation*}
if
\begin{enumerate}
\item $K_{\vH}\left(\vx\right)$ is continuous,
\item and
  \eqref{eqn:KK1},\eqref{eqn:KK2},\eqref{eqn:KK3},\eqref{eqn:KK4} are
  satisfied,
\item and $\lim\limits_{n\to\infty} \beta_i = 0, \forall i=1\dots n$.
\end{enumerate}
\end{lemma}

\begin{proof}
By the Triangle Inequality,
\begin{align}
\left|f_n^{NE}\left(\vx^i\right) - f\left(\vx^i\right)\right| &\leq \left|f_n^{NE}\left(\vx^i\right) - \tilde{f}_n^{KDE}\left(\vx^i\right)\right| + \left|f_n^{KDE}\left(\vx^i\right) - f\left(\vx^i\right)\right|\nonumber\\
&\quad+ \left|\tilde{f}_n^{KDE}\left(\vx^i\right) - f_n^{KDE}\left(\vx^i\vert\vx^{1:n}\right)\right|.\nonumber
\end{align}
Fix $\xi,\zeta>0$.  Find $N_1$ using Corollary
\ref{coro:convnekde},$N_2$ using Theorem \ref{thm:realkde}, $N_3$ by
the schedule of $\beta$ s.t.
\begin{align*}
P\left(\left|f_n^{NE}\left(\vx^i\right) - \tilde{f}_n^{KDE}\left(\vx^i\right)\right| > \xi/3\right) &< \zeta/3,\\
P\left(\left|f_n^{KDE}\left(\vx^i\right) - f\left(\vx^i\right)\right| > \xi/3\right) &< \zeta/3,\\
P\left(\left|\tilde{f}_n^{KDE}\left(\vx^i\right) - f_n^{KDE}\left(\vx^i\vert\vx^{1:n}\right)\right| > \xi/3\right) &< \zeta/3,
\end{align*}
for all $n\geq N_1$, $n\geq N_2$, and $n\geq N_3$, respectively. Set
$N=\max\left\{N_1,N_2,N_3\right\}$.  Then by Boole Inequality, when $n>N$,
\begin{align*}
P\left(\left|f_n^{NE}\left(\vx^i\right) - f\left(\vx^i\right)\right| > \xi\right) &\leq P\left(\left|f_n^{NE}\left(\vx^i\right) - \tilde{f}_n^{KDE}\left(\vx^i\right)\right| > \xi/3\right)\\
&\quad + P\left(\left|f_n^{KDE}\left(\vx^i \vert \vx^{1:n}\right) - f\left(\vx^i\right)\right| > \xi/3\right)\\
&\quad + P\left(\left|\tilde{f}_n^{KDE}\left(\vx^i\right) - f_n^{KDE}\left(\vx^i \vert \vx^{1:n}\right)\right| > \xi/3\right) < \zeta.
\end{align*}
So $\forall \xi>0$, $\lim_{n\to\infty}P\left(\left|f_n^{NE}\left(\vx^i\right) - f\left(\vx^i\right)\right| > \xi\right)=0$.
\end{proof}

\begin{lemma}\label{lemma:fillcube}
Given any uniformly continuous probability density function $f: \sX \to \RR$, and $n$ samples $\vx^1 \dots \vx^n \iid f$. $\forall \xi > 0$, if we draw a new sample $\vx \sim f$,
\[\lim\limits_{n\to\infty}P\left(\left|f\left(\vx\right)-f\left(\vx^{\argmin_{i=1..n}\normof{\vx-\vx^i}}\right)\right|>\xi\right)=0.\]
\end{lemma}

\begin{proof}
Since $f$ is uniformly continuous, given any $\xi > 0$, there exists $\varepsilon > 0$, s.t. $\forall \vx,\vy \in \sX$ and $\normof{\vx-\vy} <\varepsilon$, we have $\left| f\left(\vx\right) - f\left(\vy\right)\right|<\xi$. Therefore, let $A$ be the event $\left|f\left(\vx\right)-f\left(\vx^{\argmin_{i=1..n}\normof{\vx-\vx^i}}\right)\right|>\xi$, let $B$ be the event $\left\|\vx-\vx^{\argmin_{i=1..n}\normof{\vx-\vx^i}}\right\|>\varepsilon$ then $A\subset B$. Thus $P(A) < P(B)$, i.e. 
\begin{equation}
P\left(\left|f\left(\vx\right)-f\left(\vx^{\argmin_{i=1..n}\normof{\vx-\vx^i}}\right)\right|>\xi\right)
 < P\left(\left\|\vx-\vx^{\argmin_{i=1..n}\normof{\vx-\vx^i}}\right\|>\varepsilon\right).\nonumber
\end{equation}

If we divide $\sX$ into countable number of hypercubes with each side length being $\varepsilon$, and index the hypercubes as $\sQ_1 \dots \sQ_k \dots$. Let $p_k = \int_{\vx \in \sQ_k} f\left(\vx\right) \diff \vx$. Then the probability for event $B$ to happen is that $\vx$ is the first sample to drop in $\sQ_k, \forall k \in \ZZ$. That is, $P_n\left(B\right) \triangleq \sum\limits_{k\in\ZZ} p_k \left(1-p_k\right)^n$. Because $p_k\left(1-p_k\right)^n \to 0$, and is monotonically decreasing, using monotone convergence theorem, we have $\lim\limits_{n\to\infty} P_n\left(B\right) = 0$.

Therefore,

\begin{align}
\lim\limits_{n\to\infty}&P\left(\left|f\left(\vx\right)-f\left(\vx^{\argmin_{i=1..n}\normof{\vx-\vx^i}}\right)\right|>\xi\right)
< \lim\limits_{n\to\infty}P\left(\left\|\vx-\vx^{\argmin_{i=1..n}\normof{\vx-\vx^i}}\right\|>\varepsilon\right) \nonumber \\
& = 0.\nonumber
\end{align}
\end{proof}

Now we are ready to prove Theorem \ref{thm:nef} with Lemma \ref{lemma:convatsinglepoint} and Lemma \ref{lemma:fillcube}.

\noindent{\bf Theorem \ref{thm:nef}} Given a probability density function
$f\left(\vx\right):\sX\to\RR$, let $f_n^{NE}\left(\vx\vert
\hat{\theta}_n,\vx^{1:n}\right)\in\mathcal{NEF}_{\vs}$ be a solution satisfying
\eqref{eqn:nexp}. If
\begin{enumerate}
\setlength{\itemsep}{0em}
\setlength{\parskip}{0pt}
\setlength{\parsep}{0pt}
\item $f$ is uniformly continuous on $\sX$,
\item $K_{\vH}\left(\vx\right)$ is uniformly continuous on $\sX$,
\item and
  \eqref{eqn:KK1},\eqref{eqn:KK2},\eqref{eqn:KK3},\eqref{eqn:KK4}
  holds,
\item $\lim\limits_{n\to\infty} \beta_i = 0, \forall i=1\dots n$,
\end{enumerate}

then $f_n^{NE}\left(\vx \vert \hat{\vtheta}_n,\vx^{1:n}\right) \stackrel{p}{\to} f\left(\vx\right)$ pointwise on $\sX$.

\begin{proof}
By the Triangle Inequality, given any $\vx \in \sX$, and $f_n^{NE}\left(\vx\right)$ trained on $n$ samples $\vx^{1:n}\iid f$, let $i^\prime\triangleq \argmin_{i=1..n}\normof{\vx-\vx^i}$.
\begin{equation}
\left|f_n^{NE}\left(\vx\right) - f\left(\vx\right)\right|
\leq \left|f_n^{NE}\left(\vx^{i^\prime}\right) - f\left(\vx^{i^\prime}\right)\right| + \left|f_n^{NE}\left(\vx\right) - f_n^{NE}\left(\vx^{i^\prime}\right)\right| + \left|f\left(\vx\right) - f\left(\vx^{i^\prime}\right)\right|.\nonumber
\end{equation}
Fix $\xi,\zeta>0$, use Lemma \ref{lemma:convatsinglepoint} to find $N_1 \in \mathbb{N}$ and Lemma \ref{lemma:fillcube} to find $N_2,N_3 \in \mathbb{N}$ s.t.
\begin{align*}
P\left(\left|f_n^{NE}\left(\vx^{i^\prime}\right) - f\left(\vx^{i^\prime}\right)\right| > \xi/3\right) &< \zeta/3,\\
P\left(\left|f_n^{NE}\left(\vx\right) - f_n^{NE}\left(\vx^{i^\prime}\right)\right| > \xi/3\right) &< \zeta/3,\\
P\left(\left|f\left(\vx\right) - f\left(\vx^{i^\prime}\right)\right| > \xi/3\right) &< \zeta/3,
\end{align*}
for all $n\geq N_1$, $n\geq N_2$, and $n\geq N_3$, respectively. Set
$N=\max\left\{N_1,N_2,N_3\right\}$.  Then by Boole Inequality, when $n>N$,
\begin{align*}
P\left(\left|f_n^{NE}\left(\vx\right) - f\left(\vx\right)\right| > \xi\right) &\leq P\left(\left|f_n^{NE}\left(\vx^{i^\prime}\right) - f\left(\vx^{i^\prime}\right)\right| > \xi/3\right)\\
&\quad + P\left(\left|f_n^{NE}\left(\vx\right) - f_n^{NE}\left(\vx^{i^\prime}\right)\right| > \xi/3\right) + P\left(\left|f\left(\vx\right) - f\left(\vx^{i^\prime}\right)\right| > \xi/3\right)\\
&< \zeta.
\end{align*}
So $\forall \xi>0$, $\lim_{n\to\infty}P\left(\left|f_n^{NE}\left(\vx\right) - f\left(\vx\right)\right| > \xi\right)=0$.
\end{proof}

%% file: npexp.bbl
\begin{thebibliography}{34}
\providecommand{\natexlab}[1]{#1}
\providecommand{\url}[1]{\texttt{#1}}
\expandafter\ifx\csname urlstyle\endcsname\relax
  \providecommand{\doi}[1]{doi: #1}\else
  \providecommand{\doi}{doi: \begingroup \urlstyle{rm}\Url}\fi

\bibitem[Bach et~al.(2011)Bach, Jenatton, Mairal, and Obozinski]{Bach2011}
F.~Bach, R.~Jenatton, J.~Mairal, and G.~Obozinski.
\newblock Convex optimization with sparsity-inducing norms.
\newblock In S.~Sra, S.~Nowozin, and S.~J. Wright., editors, \emph{Optimization
  for Machine Learning}, chapter~2. MIT press, 2011.

\bibitem[Barndorff-Nielsen(1978)]{Barndorff-Nielsen1978}
O.~Barndorff-Nielsen.
\newblock \emph{Information and Exponential Families in Statistical Theory}.
\newblock Wiley, New York, 1978.

\bibitem[Bellman(1957)]{Bellman57}
R.~E. Bellman.
\newblock \emph{Dynamic Programming}.
\newblock Princeton University Press, 1957.

\bibitem[Brown(1986)]{Brown1986}
L.~D. Brown.
\newblock \emph{Fundamentals of Statistical Exponential Families}, volume~9 of
  \emph{Lecture Notes -- Monograph Series}.
\newblock Institute of Mathematical Statistics, Hayward, CA, 1986.

\bibitem[{Caimo} and {Friel}(2010)]{Caimo2010}
A.~{Caimo} and N.~{Friel}.
\newblock {Bayesian inference for exponential random graph models}.
\newblock \emph{ArXiv e-prints}, July 2010.

\bibitem[Dudik et~al.(2007)Dudik, Phillips, and Schapire]{Dudik2007}
M.~Dudik, S.J. Phillips, and R.E. Schapire.
\newblock Maximum entropy density estimation with generalized regularization
  and an application to species distribution modeling.
\newblock \emph{Journal of Machine Learning Research}, 8:\penalty0 1217--1260,
  Jun 2007.

\bibitem[Frank and Strauss(1986)]{FrankStrauss}
O.~Frank and D.~Strauss.
\newblock Markov graphs.
\newblock \emph{Journal of the American Statistical Association}, 81\penalty0
  (395), September 1986.

\bibitem[Geyer and Thompson(1992)]{Geyer1992}
C.~J. Geyer and E.~A. Thompson.
\newblock Constrained monte carlo maximum likelihood for dependent data.
\newblock \emph{Journal of the Royal Statistical Society}, 54\penalty0
  (3):\penalty0 pp. 657--699, 1992.

\bibitem[Gleiser and Danon(2003)]{Gleiser2003}
P.~Gleiser and L.~Danon.
\newblock List of edges of the network of jazz musicians.
\newblock \emph{Adv. Complex Syst. 6,565}, 2003.

\bibitem[Goodreau(2007)]{Goodreau2007}
S.~Goodreau.
\newblock {Advances in exponential random graph (p) models applied to a large
  social network}.
\newblock \emph{Social Networks}, 29\penalty0 (2):\penalty0 231--248, May 2007.
\newblock ISSN 03788733.

\bibitem[Guimera et~al.(2003)Guimera, Danon, Diaz-Guilera, Giralt, and
  Arenas]{Guimera2003}
R.~Guimera, L.~Danon, A.~Diaz-Guilera, F.~Giralt, and A.~Arenas.
\newblock List of edges of the network of e-mail interchanges between members
  of the univeristy rovira i virgili (tarragona).
\newblock \emph{Physics Review E}, 68\penalty0 (065103(R)), 2003.

\bibitem[Handcock(2003)]{Handcock2003b}
M.~S. Handcock.
\newblock Assessing degeneracy in statistical models of social networks.
\newblock Technical Report~39, Center for Statistics and the Social Sciences,
  University of Washington, 2003.

\bibitem[Handcock et~al.(2008)Handcock, Hunter, Butts, Goodreau, and
  Morris]{Handcock2008}
M.~S. Handcock, D.~R. Hunter, C.~T. Butts, S.~M. Goodreau, and M.~Morris.
\newblock {\tt statnet}: Software tools for the representation, visualization,
  analysis and simulation of network data.
\newblock \emph{Journal of Statistical Software}, 24\penalty0 (3), 2008.

\bibitem[Harris(2008)]{Harris2008}
K.~Harris.
\newblock The national longitudinal study of adolescent health (addhealth),
  waves i \& ii, 1994--1996; wave iii, 2001--2002, 2008.

\bibitem[Holland and Leinhardt(1981)]{HollandLeinhardt1981}
P.~W. Holland and S.~Leinhardt.
\newblock An exponential family of probability distributions for directed
  graphs (with discussion).
\newblock \emph{American Statistical Association}, 76\penalty0 (373), 1981.

\bibitem[Hunter and Handcock(2006)]{Hunter2006}
D.~R. Hunter and M.S. Handcock.
\newblock Inference in curved exponential family models for networks.
\newblock \emph{ASA, Journal of Computational and Graphical Statistics},
  15\penalty0 (2), 2006.

\bibitem[Hunter et~al.(2008)Hunter, Handcock, Butts, Goodreau, and
  Morris]{Hunter2008a}
D.~R. Hunter, M.~S. Handcock, C.~T. Butts, S.~M. Goodreau, and M.~Morris.
\newblock {\tt ergm}: A package to fit, simulate and diagnose
  exponential-family models for networks.
\newblock \emph{Journal of Statistical Software}, 24\penalty0 (3), 2008.

\bibitem[Jin and Liang(2012)]{Jin2011}
I.~H. Jin and F.~Liang.
\newblock Fitting social network models using varying truncation stochastic
  approximation {MCMC} algorithm.
\newblock \emph{Journal of Computational and Graphical Statistics}, In Press,
  2012.

\bibitem[Lunga and Kirshner(2011)]{Lunga2011}
D.~Lunga and S.~Kirshner.
\newblock Generating similar graphs from spherical features.
\newblock In \emph{Ninth Workshop on Mining and Learning with Graphs (MLG'11)},
  San Diego, CA, August 2011.

\bibitem[Lusseau et~al.(2003)Lusseau, Schneider, Boisseau, Haase, Slooten, and
  M.]{Lusseau2003}
D.~Lusseau, K.~Schneider, O.~J. Boisseau, P.~Haase, E.~Slooten, and Dawson~S.
  M.
\newblock The bottlenose dolphin community of {D}oubtful {S}ound features a
  large proportion of long-lasting associations.
\newblock \emph{Behavioral Ecology and Sociobiology}, 54, 2003.

\bibitem[Moreno and Neville(2009)]{Moreno2009}
S.~Moreno and J.~Neville.
\newblock {An Investigation of the Distributional Characteristics of Generative
  Graph Models}.
\newblock In \emph{Proceedings of the 1st Workshop on Information in Networks},
  February 2009.

\bibitem[Nadaraya(1965)]{Nadaraya1965}
\'{E}.~A. Nadaraya.
\newblock {On Non-Parametric Estimates of Density Functions and Regression
  Curves}.
\newblock \emph{Theory of Probability and its Applications}, 10:\penalty0
  186--190, 1965.

\bibitem[Nocedal and Wright(2006)]{NocedalWright06}
J.~Nocedal and S.~J. Wright.
\newblock \emph{Numerical Optimization}.
\newblock Springer Series in Operations Research. Springer, 2nd edition, 2006.

\bibitem[Parzen(1962)]{Parzen1962}
E.~Parzen.
\newblock {On Estimation of a Probability Density Function and Mode}.
\newblock \emph{The Annals of Mathematical Statistics}, 33\penalty0
  (3):\penalty0 1065--1076, 1962.

\bibitem[Rinaldo et~al.(2009)Rinaldo, Fienberg, and Zhou]{Rinaldo2009}
A.~Rinaldo, S.~E. Fienberg, and Y.~Zhou.
\newblock {On the geometry of discrete exponential families with application to
  exponential random graph models}.
\newblock \emph{Electronic Journal of Statistics}, 3:\penalty0 446--484, 2009.

\bibitem[Robins et~al.(2007{\natexlab{a}})Robins, Pattison, Kalish, and
  Lusher]{Robins2007a}
G.~Robins, P.~Pattison, Y.~Kalish, and D.~Lusher.
\newblock An introduction to exponential random graph (p*) models for social
  networks.
\newblock \emph{Social Networks}, 29\penalty0 (2):\penalty0 173 -- 191,
  2007{\natexlab{a}}.

\bibitem[Robins et~al.(2007{\natexlab{b}})Robins, Snijders, Wang, Handcock, and
  Pattison]{Robins2007b}
G.~Robins, T.~Snijders, P.~Wang, M.~Handcock, and P.~Pattison.
\newblock Recent developments in exponential random graph (p*) models for
  social networks.
\newblock \emph{Social Networks}, 29\penalty0 (2):\penalty0 192 -- 215,
  2007{\natexlab{b}}.

\bibitem[Royden(1988)]{Royden1988}
H.~L. Royden.
\newblock \emph{Real Analysis}.
\newblock Prentice Hall, 3rd edition, 1988.

\bibitem[Shalev-Shwartz and Tewari(2011)]{Shalev-ShwartzTewari2011}
S.~Shalev-Shwartz and A.~Tewari.
\newblock Stochastic methods for $l_1$ regularized loss minimization.
\newblock \emph{Journal of Machine Learning Research}, 12:\penalty0 1865--1892,
  June 2011.

\bibitem[van~der Vaart(1998)]{Vaart1998}
A.~W. van~der Vaart.
\newblock \emph{Asymptotic Statistics}.
\newblock Cambridge Series in Statistical and Probabilistic Mathematics.
  Cambridge University Press, 1998.

\bibitem[Wainwright and Jordan(2008)]{Wainwright2008}
M.~J. Wainwright and M.~I. Jordan.
\newblock Graphical models, exponential families, and variational inference.
\newblock \emph{Foundations and Trends in Machine Learning}, 1\penalty0
  (1-2):\penalty0 1--305, 2008.

\bibitem[Wasserman and Pattison(1996)]{WassermanPattison}
S.~Wasserman and P.~Pattison.
\newblock Logit models and logistic regression for social networks: An
  introduction to {M}arkov graphs and p* model.
\newblock \emph{Psychometrii}, 61\penalty0 (3), September 1996.

\bibitem[Wasserman and Robins(2004)]{WassermanRobins2004}
S.~Wasserman and G.~Robins.
\newblock An introduction to random graphs, dependence graphs, and p*,.
\newblock In P.~J. Carrington, J.~Scott, and S.~Wasserman, editors,
  \emph{Models and Methods in Social Network Analysis}. Cambridge University
  Press, 2004.

\bibitem[Wu and Lange(2008)]{WuLange2008}
T.~Wu and K.~Lange.
\newblock Coordinate descent algorithms for lasso penalized regression.
\newblock \emph{The Annals of Applied Statistics}, 2\penalty0 (1):\penalty0
  224--244, March 2008.

\end{thebibliography}
